\documentclass{article}

\usepackage{arxiv}

\usepackage[utf8]{inputenc} % allow utf-8 input
\usepackage[T1]{fontenc}    % use 8-bit T1 fonts
\usepackage{hyperref}       % hyperlinks
\usepackage{url}            % simple URL typesetting
\usepackage{booktabs}       % professional-quality tables
\usepackage{amsfonts}       % blackboard math symbols
\usepackage{nicefrac}       % compact symbols for 1/2, etc.
\usepackage{microtype}      % microtypography
\usepackage{lipsum}		% Can be removed after putting your text content
\usepackage{graphicx}
\usepackage[numbers]{natbib}
\usepackage{doi}
\usepackage{subcaption}
\usepackage{setspace}
\usepackage{xcolor}
\usepackage[linesnumbered, ruled, vlined]{algorithm2e}
\usepackage{hyperref}
\usepackage[titletoc,title]{appendix}

\usepackage{lineno}

\usepackage{amsmath}
\usepackage{amsthm}

\SetKwInput{KwInput}{Input}
\SetKwInput{KwProcedure}{Procedure}

\newcommand{\Xv}{\mbox{\boldmath$X$}}
\newcommand{\xv}{\mbox{\boldmath$x$}}

\newcommand{\piv}{\mbox{\boldmath$\pi$}}

\newcommand{\lv}{{\bf 1}}

\newcommand{\pmis}{P_{\rm mc}}

\newtheorem{prop}{Proposition}

\title{Interpretable Clustering with the Distinguishability Criterion}

\author{ Ali Turfah \\
	Department of Biostatistics\\
	University of Michigan\\
	Ann Arbor, MI, 48105 \\
	\texttt{aturfah@umich.edu} \\
	\And
	Xiaoquan Wen\thanks{Corresponding Author} \\
	Department of Biostatistics\\
	University of Michigan\\
	Ann Arbor, MI, 48105 \\
	\texttt{xwen@umich.edu} \\
}

\date{April 24, 2024}

\renewcommand{\shorttitle}{Interpretable Clustering with the Distinguishability Criterion}

\hypersetup{
pdftitle={Interpretable Clustering with the Distinguishability Criterion},
pdfsubject={q-bio.NC, q-bio.QM},
pdfauthor={Ali Turfah, Xiaoquan Wen},
pdfkeywords={Unsupervised learning, Cluster analysis, k-means, Hierarchical clustering, Mixture models},
}

\begin{document}
% \linenumbers

\maketitle

\begin{abstract}
    Cluster analysis is a popular unsupervised learning tool used in many disciplines to identify heterogeneous sub-populations within a sample. 
    However, validating cluster analysis results and determining the number of clusters in a data set remains an outstanding problem. 
    In this work, we present a global criterion called the Distinguishability criterion to quantify the separability of identified clusters and validate inferred cluster configurations. 
    Our computational implementation of the Distinguishability criterion corresponds to the Bayes risk of a randomized classifier under the 0-1 loss. 
    We propose a combined loss function-based computational framework that integrates the Distinguishability criterion with many commonly used clustering procedures, such as hierarchical clustering, $k$-means, and finite mixture models. 
    We present these new algorithms as well as the results from comprehensive data analysis based on simulation studies and real data applications.
\end{abstract}

% keywords can be removed

\section{Introduction}

Cluster analysis is a ubiquitous unsupervised learning approach to uncover latent structures and patterns in observed data. 
Clustering algorithms have been used in a wide variety of scientific applications, such as animal behavior studies \cite{braun2010identifying}, weather anomaly detection \cite{wibisono2021multivariate}, disease diagnosis \cite{ahmad2015techniques, kao2017spatial, shafqat2020big}, and novel cell type identification \cite{xie2018qubic2, kiselev2019challenges, kanter2019cluster, qi2020clustering}. 
Often, the identified clusters are interpreted to represent distinct populations from which the corresponding samples originate.

Many challenges with cluster analysis, such as determining the number of clusters, arise from an inability to rigorously quantify desired cluster characteristics. 
While the precise definition of a ``meaningful'' cluster is usually context-dependent, it is generally accepted that the clusters should display ``internal cohesion'' (i.e., objects within a cluster are similar to one another) and ``external isolation'' (i.e., the clusters are well-separated) \cite{wolfe1963object, cormack1971review, oyewole2023data}. 
Despite this almost universally agreed-upon principle, quantifying the separability of the clusters, i.e., the level of external isolation with respect to internal cohesion, remains an open problem in cluster analysis.

In this paper, we introduce the Distinguishability criterion to measure the separability of a set of assumed clusters. 
The criterion is motivated by this simple intuition: if all clusters are well separated from each other, then the originating clusters for all data points (whether observed or not) should be easily traceable.
To implement the Distinguishability criterion, we formulate labeling the generating cluster for an arbitrary data point as a probabilistic classification problem.
Naturally, the difficulty (or lack thereof) of this classification problem can be described by an overall misclassification probability averaged over all possible data points.

We employ a statistical viewpoint to define the Distinguishability criterion for cluster analysis.
The partitioned observed data are taken to be realizations from cluster-specific data generative distributions, which are essential for computing the proposed misclassification probability.
Although not all clustering algorithms make explicit distributional assumptions for the presumed clusters, many commonly applied heuristics-based algorithms achieve optimal performance under specific probabilistic generative models 
\cite{Bock_1996, Fraley_2002}.
Furthermore, the identified cluster structures from cluster analysis are typically expected to be replicated in future datasets---an implicit assumption for consistent data generative distributions. 
As a result, statistical inference procedures based on explicit distributional assumptions have become more popular for their ability to not only enhance clustering performance but also to examine and interpret cluster structures identified by both model and heuristics-based clustering methods \cite{Kimes_2017, gao2022selective, chen2022selective, Grabski_2023}.

The remainder of this paper is organized as follows. 
We first provide the mathematical definition and properties of the Distinguishability criterion. 
We then discuss its usage with existing clustering algorithms. 
Finally, we illustrate the applications of the Distinguishability criterion using both synthetic and real data from various scientific applications. 
Our implementation of the Distinguishability criterion and the analyses presented in this paper can be found at \texttt{\href{https://github.com/aturfah/distinguishability-criterion}{https://github.com/aturfah/distinguishability-criterion}}.

\section{Results}

\subsection{The Distinguishability Criterion} \label{sec:distinguishability_criterion}

The proposed Distinguishability criterion measures the overall separability of a given cluster configuration and is derived by quantifying the misclassification probability from a multi-class classification problem.

Given a cluster configuration where each of the $K$ clusters corresponds to a distinct class, we denote the class label for an observation $\xv$ by $ \theta \in \{ 1,2,..., K\}$. We assume a pre-defined classifier, $\delta(\xv): \mathbb{R}^p \to \{1,...,K\}$, and evaluate the classification performance using the 0-1 loss function, i.e.,

\begin{equation*}
    L(\delta(\xv), \theta) = \lv \{\delta(\xv) \ne \theta \}.
\end{equation*}  

The overall misclassification probability under the assumed cluster configuration, denoted by $\pmis$, is defined as the Bayes risk of the classifier $\delta(\xv)$, i.e.,

\begin{equation}
    \pmis = \mathrm{E}_{\xv} \Big[ \, \mathrm{E}_{\theta} \big( \, L(\delta(\xv), \theta) \mid  \xv \,\big) \, \Big] 
\end{equation}

Using the 0-1 loss ensures that the resulting Bayes risk is a valid probability measurement, ranging from 0 to 1. 
It is naturally interpreted as the probability of misclassifying a data point under the given cluster configuration, marginalizing all potential $\xv$ values and their respective true generating clusters. 
For instance, a $\pmis$ value close to 0 signifies a high degree of cluster separation, indicated by a minimal probability of erroneously assigning an arbitrary data point to an incorrect generating cluster (Supplementary Figure 2).

As we are primarily interested in assessing different cluster configurations, the selection of the required classifier is flexible. 
However, the choice of classifier can impact the computational efficiency of the $\pmis$ evaluation.
Our implementation focuses on the set of classifiers working directly with the probabilities

\begin{equation*}
    \pi_k (\xv) := \Pr(\theta = k \mid \xv), ~\mbox{ for } k = 1,...,K.
\end{equation*}

This set includes the optimal classifier under the 0-1 loss, $\delta_o$, i.e.,

\begin{equation*}
    \delta_o(\xv) = \arg\max_k \pi_k(\xv)
\end{equation*}

Our default classifier for computing $\pmis$ is a randomized decision function, $\delta_r$, which assigns a label to an observation $\xv$ by sampling from a categorical distribution based on the probability distribution $\piv (\xv) = \left(\pi_1 (\xv),...,\pi_K(\xv)\right)$, i.e.,

\begin{equation*}
    \delta_r(\xv) \sim {\rm Categorical}\, (\piv(\xv)).
\end{equation*}

In addition to yielding highly comparable $\pmis$ values to the optimal classifier within the decision-critical ranges of cluster separation (Supplementary Figure 1), the randomized classifier's computational properties enable highly efficient cluster analysis procedures.

The $\pi_k$'s are the key quantities bridging the observed clustering data and $\pmis$. They are calculated using Bayes rule,

\begin{equation*}
    \pi_k (\xv) \, \propto \, \alpha_k(\Xv_c) \, p(\xv \mid \theta(\Xv_c) = k).
\end{equation*} 

The notation emphasizes that both the prior, $\alpha_k(\Xv_c)$, and the likelihood function, $p(\xv \mid \theta(\Xv_c))$, are directly informed by and estimated from the clustering data.
More specifically, the prior quantifies the relative frequency of the observations arising from each assumed cluster, while the likelihood function encodes the characteristics of the corresponding cluster population, such as its centroid and spread information.
Computing $\pi_k$ values is straightforward for model-based clustering algorithms such as Gaussian mixture models (GMMs) \cite{Fraley_2002}.
For non-model-based clustering algorithms, explicit distributional assumptions specifying the parametric family of likelihood functions are required.
We illustrate these procedures for $k$-means and hierarchical clustering algorithms in subsequent sections. 

In summary, $\pmis$ is a probability measurement of global separability across inferred clusters. 
It can accommodate a wide range of distributional assumptions, making it compatible with a diverse set of clustering procedures and data modalities. 
Moreover, as a function of clustering data, $\Xv_c$, the estimate of $\pmis$ itself is a valid loss function suitable for selecting optimal cluster configurations in cluster analysis. 

\subsection{Combined Loss Function for Cluster Analysis} \label{sec:combined_loss}

In cluster analysis, the desired cluster characteristics are often defined by multiple criteria \cite{Hennig_2015}. A single criterion on its own, including the proposed Distinguishability criterion, is insufficient to define a practically optimal clustering solution. 
Alternatively, combining multiple loss functions targeting different desired cluster properties can result in more balanced and holistic clustering solutions. 
This observation leads to a principled way to incorporate $\pmis$ with other established clustering criteria and algorithms. 

Specifically, let $L_1$ denote a loss function associated with existing clustering algorithms. Formally, we consider a compound loss, $L$, as a weighted linear combination of $L_1$ and $\pmis$, i.e.,

\begin{equation}
    L = L_1 + \lambda \pmis, ~ \lambda >  0.
\end{equation} 

Because of the scale of $\pmis$, it is often convenient to solve the following equivalent constrained optimization problem,

\begin{equation}\label{combined.loss2}
    \mbox{Minimize $L_1$,  subject to } \pmis \le \tau,  
\end{equation}

where $\tau$ is a pre-specified probability threshold. It is worth noting that the stringency of the $\tau$ value may depend on the dimensionality of the clustering data.

The choice of clustering algorithm determines the functional form of $L_1$. For example, the distortion function or Ward's linkage are natural choices for $L_1$ when using $k$-means and hierarchical clustering methods, respectively. 
Alternatively, the negative of the gap statistic \cite{tibshirani2001estimating} can also be an excellent choice in these application scenarios.
For model-based clustering algorithms, the $L_1$ function can be derived from various model selection criteria, e.g., the negative of Bayesian information criterion (BIC). 

\subsection{Connections to Related Approaches}

The misclassification probability defined by $\pmis$ falls into the category of internal clustering validity indices \cite{Halkidi_2001, Kim_2005, Liu_2010}, which assess the quality of a clustering solution without additional external information beyond the observed data.
This category includes many commonly applied statistical measures, such as the Silhouette index \cite{rousseeuw1987silhouettes}, Calinski-Harabaze index \cite{calinski1974dendrite}, Dunn index \cite{dunn1974well}, among others. 
A common behavior of internal clustering validity indices is that they evaluate both the cohesion (or compactness) within a cluster as well as the separation between clusters. 
For $\pmis$, the within-cluster cohesion is quantified through the estimated parameters in the likelihood function, $p(\xv \mid \theta(\xv_c) = k)$.
The separation, relative to the cohesion, is quantified by the overall misclassification probability.

Henning \cite{Hennig_2015} and Melnykov \cite{Melnykov_2016} also compute misclassification probabilities to assess the separation between clusters in the context of mixture models for clustering.
They introduce the metrics---named ``directly estimated misclassification probability" (DEMP) and DEMP+---specifically designed to compute misclassification probabilities between pairs of clusters to inform decisions about the local merging of mixture components.
In comparison, $\pmis$ is a global measure of the misclassification probability across all clusters. It, too, can be used to combine mixture components to form interpretable clusters, as illustrated in Section \ref{sec:mixture_model_pmis}.
Additionally, the entropy criterion proposed by Celeux and Soromenho \cite{Celeux_1996, Biernacki_2000, Baudry_2010} provides another alternative approach to quantify the separability of clusters using a classification problem set-up.

The Distinguishability criterion also agrees with the principle of stability measures commonly employed in clustering analysis. Specifically, if the underlying cluster distributions are all well-separated, as indicated by low $\pmis$ values, data sampled repeatedly from these distributions are expected to produce consistent clustering outcomes \cite{Luxburg_2010,lange2004stability,tibshirani2005cluster}.
Empirical evidence has shown that $\pmis$ and various measures of clustering instability are highly correlated, which is demonstrated in the subsequent sections.

\subsection{Finite Mixture Models Incorporating $\pmis$} \label{sec:mixture_model_pmis}

Finite mixture models (MM) are probabilistic models that can seamlessly incorporate the Distinguishability criterion.
MM-based clustering algorithms primarily infer the distributional characteristics underlying each latent cluster. 
As a result, no additional assumptions are needed to compute $\pmis$ in the MM setting --- the required quantities, $\{\alpha_k, \, p( \xv \mid \theta=k), \, \pi_k\}$, are all direct outputs or by-products from standard MM inference procedures \citep{Fraley_2002}, e.g., the EM algorithm.

We make an important distinction between a mixture component and an interpretable cluster, a point previously discussed by \cite{Melnykov_2016, Biernacki_2000, Baudry_2010, Hennig_2010}.
We view mixture models as flexible density estimation devices, where the number of mixture components is chosen to adequately fit the observed data. On the other hand, the distribution of an underlying cluster---characterized by the Distinguishability criterion---may itself be a mixture distribution comprising multiple components. This distinction naturally leads to combining the loss functions represented by $-{\rm BIC}$, which evaluates the goodness-of-fit of a mixture density, and $\pmis$, which characterizes the separation between potential clusters.

To optimize the combined loss function, we first find $P(\xv)$, the optimal mixture distribution with $\kappa$ components, by maximizing the BIC. Subsequently, we merge the mixture components into clusters until $\pmis$ falls below a pre-specified threshold $\tau$. 
Since merging mixture components into clusters does not alter the mixture component distributions in any way, the BIC is unchanged by the merging process. 
It can be shown that merging existing clusters in this manner always decreases $\pmis$ (Appendix \ref{app:merge_property}). 
Specifically, under the default randomized classifier $\delta_r$, the reduction of $\pmis$ by combining clusters $i$ and $j$ is given by

\begin{equation}
    \Delta \pmis^{(i, j)} = 2 \int \pi_i(\xv) \pi_j(\xv) \, P(d \xv).
\end{equation} 

For clusters with little to no overlap, i.e., $\pi_i(\xv) \pi_j(\xv) \to 0$ for all $\xv$ values, merging $(i,j)$ results in minimal changes in $\pmis$. On the other hand, for clusters with significant overlap, $\Delta \pmis^{(i,j)}$ can be substantial.

By further utilizing the cluster merging property of $\pmis$ (Proposition \ref{prop:pmis_merge}, Methods Section), i.e.,  

\begin{equation}
    \pmis = \sum_{i<j} \Delta \pmis^{(i, j)}, 
\end{equation}

we propose an efficient $\pmis$ Hierarchical Merging (PHM) algorithm to sequentially amalgamate mixture components into clusters (Algorithm \ref{alg:pmis_gmm_algorithm}). 
Briefly, starting by assigning each of the $\kappa$ mixture components to individual clusters, the PHM algorithm pre-computes $\Delta \pmis^{(i, j)}$ for all $(i, j)$ cluster pairs. 
It then sequentially combines the pairs of clusters with the largest $\Delta \pmis^{(i, j)}$ into a single cluster and updates the $\Delta \pmis$ values for the remaining clusters. 
The process is repeated until the updated $\pmis$ falls at or below a pre-defined $\tau$ value. 
Intuitively, this procedure prioritizes merging the most similar or closely related clusters at each step, quantified by their $\Delta \pmis^{(i, j)}$ value.

\begin{algorithm}[b!]
    \DontPrintSemicolon
    \caption{$\pmis$ Hierarchical Merging (PHM) algorithm}
    \label{alg:pmis_gmm_algorithm}
    \KwIn{Input data $X$, $\pmis$ threshold $\tau$}
    \KwResult{Groupings of mixture components into clusters}
    \textbf{Procedure} PHM\;
    \Indp
    Fit a mixture model to $X$, determining the number of components by maximizing BIC\;
    Initialize clusters to individual mixture components\;
    Compute $\Delta \pmis^{(i, j)}$ for all pairs of clusters $i, j$\;
    \While {$\pmis > \tau$} {
        Group clusters $i,j$ with maximal $\Delta \pmis^{(i, j)}$ into a single cluster $k'$\;
        Update the distribution quantities for this new cluster: \;
            \Indp $\alpha_{k'} \gets \alpha_j + \alpha_i$ \;
            $p( \xv \mid \theta=k') \gets \alpha_{k'}^{-1} \cdot \left[ \, \alpha_i \cdot p( \xv \mid \theta=i) + \alpha_j \cdot p( \xv \mid \theta=j) \, \right]$ \;
            $\pi_{k'}(\xv) \gets \pi_{i}(\xv) + \pi_{j}(\xv)$  \;
        \Indm Update $\pmis \gets \pmis  - \Delta \pmis^{(i, j)}$ \;
        Compute $\Delta \pmis^{(k', k)}$ for all uninvolved clusters $k$: $\Delta \pmis^{(k', k)} = \Delta \pmis^{(i, k)} + \Delta \pmis^{(j, k)}$ \;
    }
    \Return 
\end{algorithm}

By setting $\tau = 0$, the algorithm runs until all mixture components have been merged into a single cluster. 
The complete merging process can be visualized using a dendrogram (Appendix \ref{app:dendrogram}), characterizing the hierarchical merging orders between individual mixture components and merged clusters. 
In many scientific applications, e.g., genetics and single-cell data analysis, such a dendrogram can provide a snapshot of the underlying continuous differentiation process that forms the identified clusters. We provide two examples in our real data applications (Section \ref{sec:real_data}).

To illustrate the PHM algorithm with Gaussian mixture models (GMMs), we use the synthetic data from Section 4.1 in Baudry et al. \citep{Baudry_2010}.
Specifically, 600 observations are drawn from six Gaussian distributions arranged along the corners of a square in the following manner. 
Two overlapping Gaussian distributions are placed in the top left corner of the square, each with 1/5 of the samples. 
The bottom left and top right corners each have a single Gaussian distribution, each contributing 1/5 of the samples. 
Finally, two overlapping Gaussian distributions are placed in the bottom right corner, each with 1/10 of the samples.

A GMM with six components is selected by BIC using the \texttt{R} package \texttt{mclust} \cite{mclust}. 
The observed clustering data is shown in the left panel of Figure \ref{fig:baudry_heatmap}, with colors corresponding to an observation's assignment to one of the $\kappa$ GMM components. 
The initial cluster configuration labeling each mixture component as a single cluster has $\pmis = 0.139$.
The heatmap in the right panel of Figure \ref{fig:baudry_heatmap} visualizes the $\Delta \pmis^{(i, j)}$ contributions from each pair of mixture components, showing that the main sources of $\pmis$ come from the overlapping components in the top left and bottom right corners.

\begin{figure}[t!]
    \centering
    \includegraphics[]{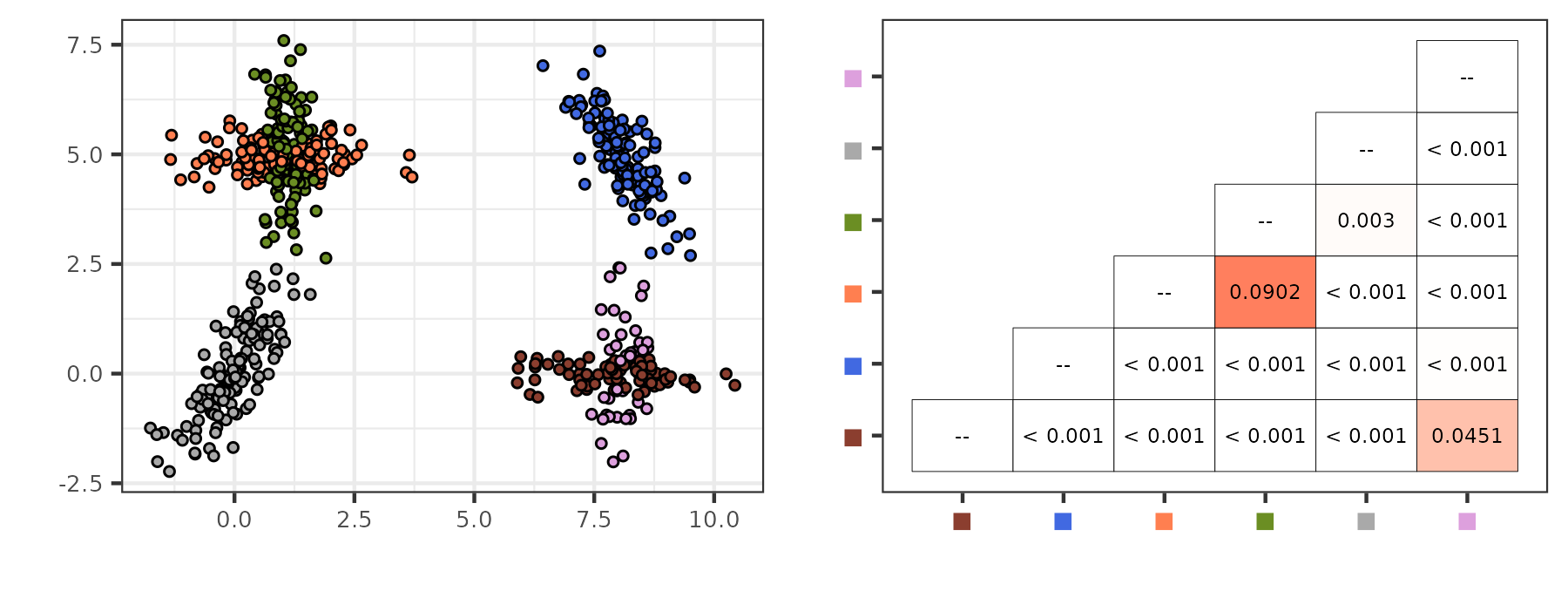}
    \caption{\textit{(Left)} 600 simulated observations drawn from a mixture of six two-dimensional Gaussian distributions. Colors indicate the cluster assignment labels to each of the six estimated mixture components. \textit{(Right)} Heatmap visualizing $\Delta \pmis$ values for the estimated mixture components. The intensity of the color indicates the relative proportion of $\pmis$ contributed by the overlap between these components (i.e., $\Delta \pmis^{(i, j)}$).}
    \label{fig:baudry_heatmap}
\end{figure}

With a threshold of $\tau = 0.01$, the PHM algorithm sequentially combines the mixture components in the top left and bottom right corners, reducing the values of $\pmis$ to 0.049 and 0.004, respectively. 
In the end, the algorithm returns four clusters: one corresponding to the components in each of the four corners.

\subsection{Applications in $k$-means Clustering} \label{sec:kmeans_procedure}

Applying the Distinguishability criterion to heuristics-based clustering algorithms requires additional distributional assumptions to compute $\pmis$.  
Although popular algorithms of this kind---such as $k$-means and hierarchical clustering---rely on intuitive heuristics, their implicit connections to probability models have been well studied.
These results provide insights into the data types for which certain non-model-based clustering algorithms are expected to be optimal. 
The $k$-means algorithm, in particular, has been shown as equivalent to approximately optimizing a multivariate Gaussian classification likelihood function \cite{Bock_1996, Fraley_2002}.
Hence, when applying the Distinguishability criterion with the usual $k$-means distortion function, it seems natural to assume that data within each inferred partition are normally distributed.
It then becomes straightforward to estimate the necessary parameters and compute $\pmis$ given the partitioned data from the $k$-means output.

We illustrate a procedure to determine the number of clusters ($K$) in $k$-means clustering by optimizing the combined loss of $\pmis$ and the gap statistic \cite{tibshirani2001estimating}.
For a given data partitioning from the $k$-means algorithm, the gap statistic compares the observed within-cluster dispersion to the expected dispersion under a null reference distribution.
It subsequently estimates the optimal number of clusters corresponding to the largest gap statistic value from a range of potential $K$ values. 
We consider observations drawn from three Gaussian clusters in $\mathbb{R}^2$ centered at (0, 0), (1.75, 1.75), and (-4, 4) with identity covariance matrices ($n_k = 150$). The observed data are shown in Figure \ref{fig:kmeans_overlap}, with substantial overlap between the data generated from two of the three distributions.
The $k$-means algorithm is performed for $K = 1$ to 7 using the R package \texttt{ClusterR} \citep{ClusterR} with \texttt{kmeans++} initialization \citep{arthur2007k}. In addition to the $\pmis$ values, we also present the averaged Silhouette index \cite{rousseeuw1987silhouettes} for each value of $K=2, \dots, 7$.

The values of the gap statistic, $\pmis$, and the Silhouette index for different $K$ are shown in Figure \ref{fig:kmeans_overlap}. When used alone, the gap statistic selects $K = 3$, coinciding with the number of generating distributions.
There is a noticeable difference in the inferred clusters' separability between $K=2$ ($\pmis = 0.002$) and $K=3$ ($\pmis = 0.062$).
For a $\pmis$ threshold $\tau = 0.01$, optimizing the combined loss function (\ref{combined.loss2}) leads to selecting $K = 2$, which seems most reasonable judging by the data visualization.
This decision is also consistent with the Silhouette index, which takes a maximal value of 0.632 at $K= 2$ compared to 0.522 at $K = 3$.
Additionally we find that $\pmis$ is highly negatively correlated with cluster stability measures. We present the values of the stability measure proposed by Lange et al. \cite{lange2004stability} using the adjusted Rand index to compute cluster similarities as well as the prediction strength \cite{tibshirani2005cluster} in Supplementary Table 1.

As noticed in the original paper \citep{tibshirani2001estimating}, the gap statistic can struggle to determine $K$ when the underlying generative distributions have substantial overlapping support.   
The above numerical example illustrates that the Distinguishability criterion can alleviate this challenge in $k$-means clustering. 

\begin{figure}[t!]
    \centering
    \includegraphics[]{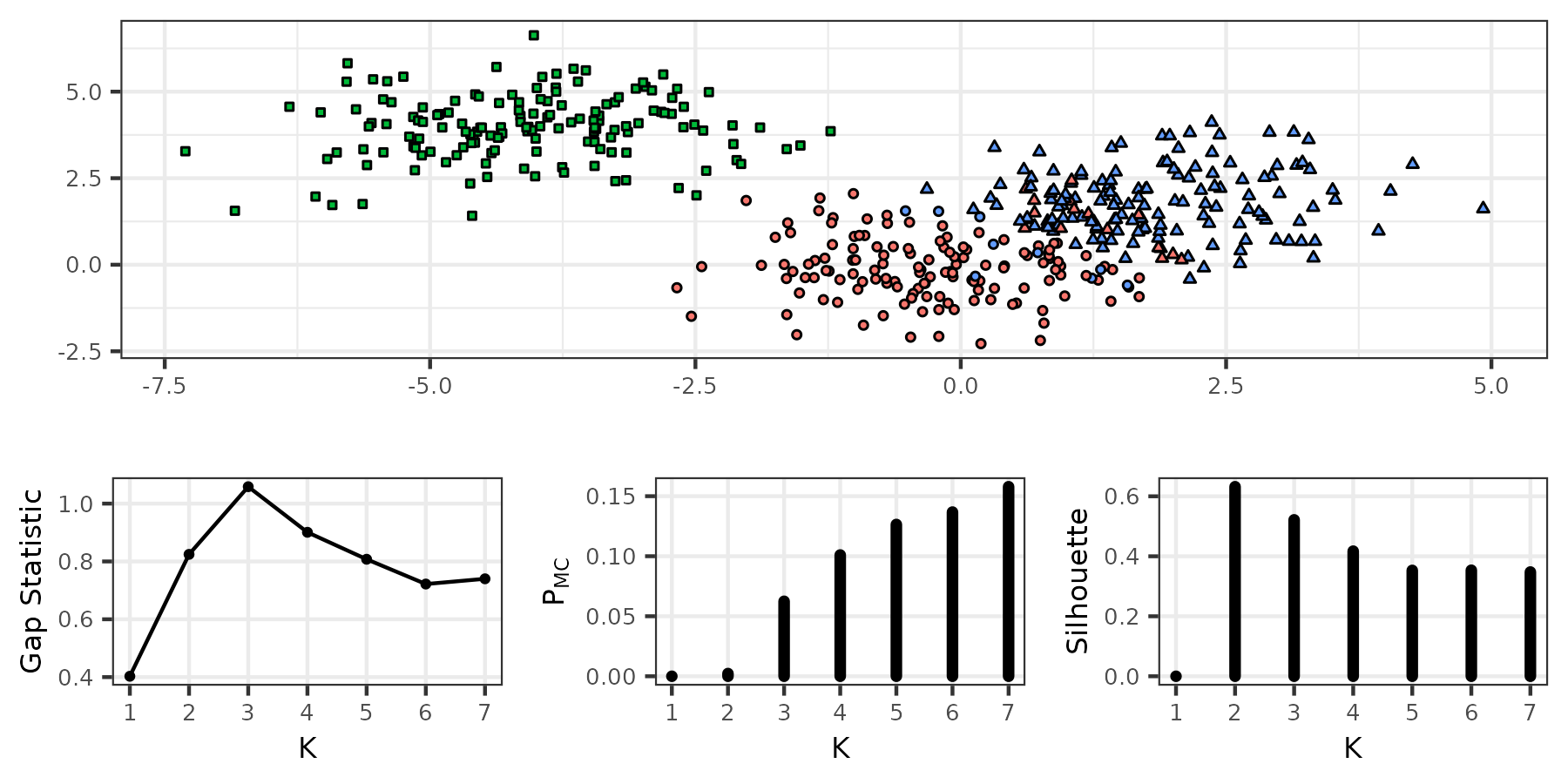}
    \caption{\textit{(Left)} 450 simulated observations drawn from a mixture of three Gaussian distributions. Color indicates true generating distribution while shape indicates the assigned $k$-means cluster. (\textit{Center and Right}) Value of the gap statistic, $\pmis$, and the Silhouette index for different numbers of clusters based on the $k$ means clustering partition with Gaussian cluster distributions.}
    \label{fig:kmeans_overlap}
\end{figure}

\subsection{Hypothesis Testing in Hierarchical Clustering}

It has become increasingly common to perform formal statistical testing in post-clustering analysis to reduce cluster over-identification. 
Recent works by Gao et al. \citep{gao2022selective}, Chen et al. \cite{chen2022selective}, and Grabski et al. \citep{Grabski_2023} highlight the necessity and importance of such analysis in scientific applications, where false positive findings of clusters are considered more costly than false negative findings. 
In many application contexts, it is often of interest to assess whether the partitions of the observed data output by heuristics-based clustering algorithms (e.g., hierarchical clustering) could arise from a single homogeneous distribution (most commonly, a single Gaussian distribution) by chance.  
$\pmis$ emerges as a natural test statistic in this parametric hypothesis testing framework because its values are expected to be quantitatively different under the null and alternative scenarios.

Formally, we consider testing the null hypothesis,

\begin{equation*}
    H_0: \mbox{the data are generated from a single Gaussian distribution}
\end{equation*}

in the hierarchical clustering setting \cite{Kimes_2017, gao2022selective,Grabski_2023}.
Following hierarchical clustering of observed data, we compute $\pmis$ for the pair of inferred clusters resulting from the first split of the dendrogram (i.e., $K=2$). 
The null distribution of $\pmis$ can be estimated by a simple Monte Carlo procedure that repeatedly samples data from $H_0$, performs hierarchical clustering, and computes $\pmis$ for the partitions defined by the first split.
Alternatively, the $p$-value of the $\pmis$ statistic can be derived from the standard parametric bootstrap procedure \cite{hinkley1988bootstrap}. 

We present the results of simulation studies examining the performance of $\pmis$ in this hypothesis testing setting.

\begin{figure}[b!]
    \centering
    \begin{subfigure}{1\textwidth}
        \centering
        \includegraphics[]{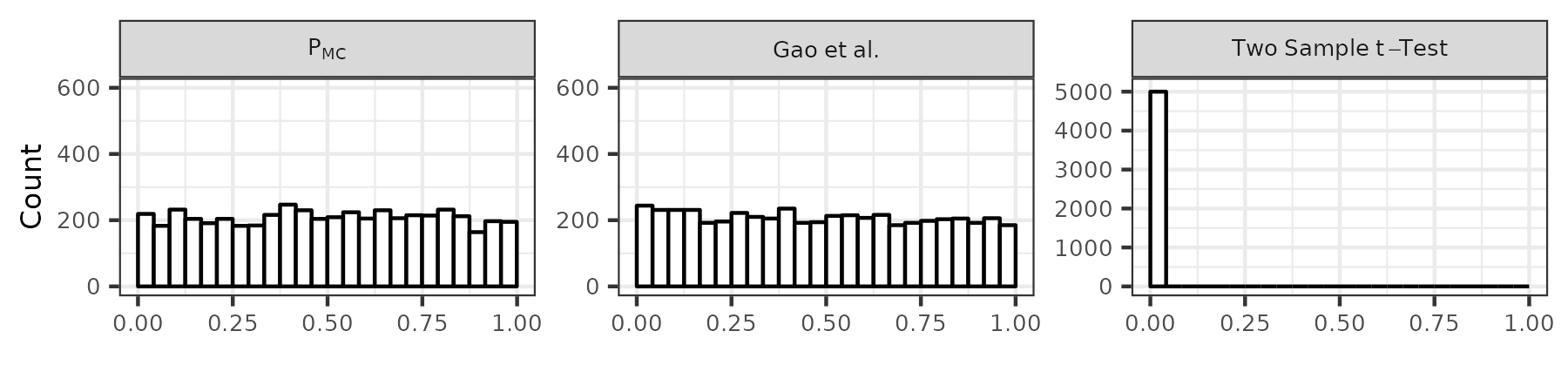}
        % \label{fig:hyp_test_distbn}
    \end{subfigure}%
    \newline
    \begin{subfigure}{1\textwidth}
        \centering
        \includegraphics{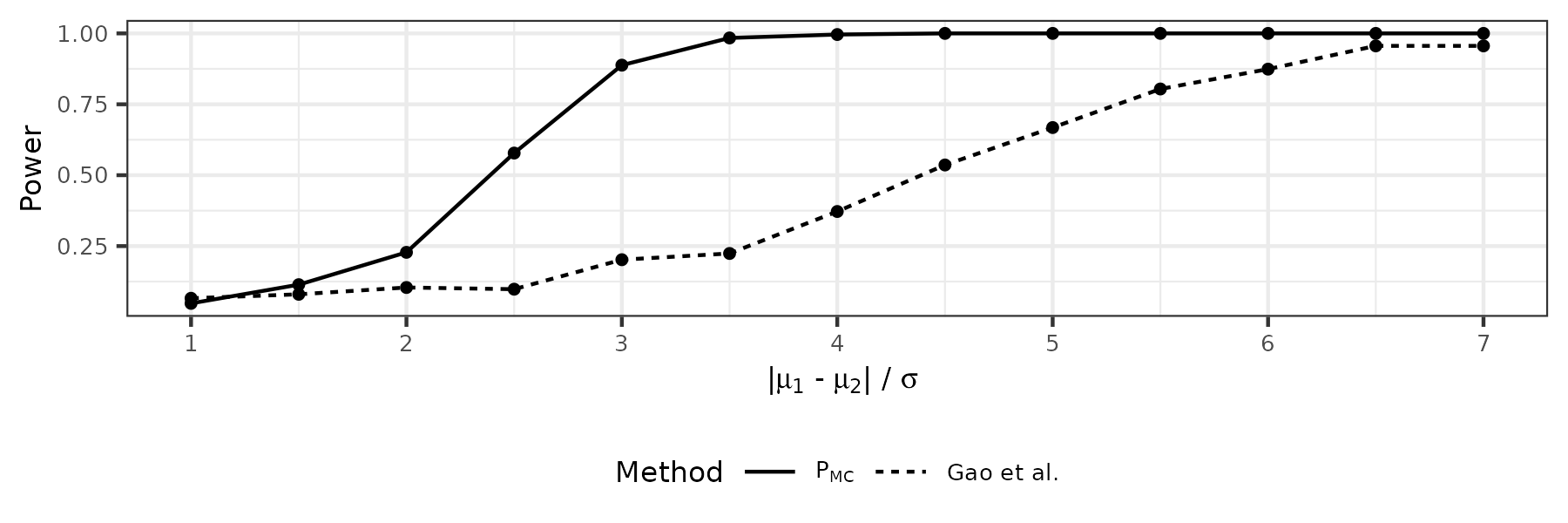}
        % \label{fig:hyp_test_power}
    \end{subfigure}
    \caption{(\textit{Top}) The distribution of $p$-values based on $\pmis$, Gao et al.'s selective inference procedure, and the two-sample t-test for 5,000 simulation replicates. (\textit{Bottom}) Power comparison between $\pmis$ and Gao et al.'s method to detect the presence of two Gaussian clusters controlling the type I error rate at level $\alpha = 0.05$ as the cluster separability increases. The power at each value $|\mu_1 - \mu_2| / \sigma$ is calculated based on 500 simulation replicates.}
    \label{fig:hyp_test}
\end{figure}

First, we illustrate $\pmis$'s ability to control for spurious cluster detection. For 5,000 simulation replicates, we draw 150 observations from a $N(0, 1)$ distribution and perform hierarchical clustering based on the squared Euclidean distance with the Ward linkage. 
We compute $\pmis$ and derive the corresponding $p$-values based on the estimated null distribution from 5,000 additional Monte Carlo simulations, each with a sample size of 150. 
We also compute the $\pmis$ $p$-values using a bootstrap procedure, using 500 bootstrap replicates to estimate the $p$-value for each simulation replicate. 
For comparison, we also compute the $p$-values from the selective inference procedure proposed by Gao et al. \citep{gao2022selective} as well as the standard two-sample $t$-test.
The results are shown in the top panel of Figure \ref{fig:hyp_test}. 
As expected, the $p$-values from the two-sample $t$-test are all in the range of $10^{-20}$ to $10^{-40}$, providing no control for false positive findings.
This is because the naive $t$-test fails to take into account that the hierarchical clustering procedure always partitions data according to their observed values even under $H_0$, as noted by \cite{gao2022selective}.
In contrast, $p$-values derived from $\pmis$ and Gao et al.'s methods are both roughly uniformly distributed under $H_0$, suggesting well-controlled type I error rates.
Specifically, we find that $\pmis = 0.094$ corresponds to the cutoff at the 5\% $\alpha$ level, and the realized type I error rate is 4.8\%.
The type I error rate from the bootstrap procedure using $\pmis$ at the same control level is 1\%, which is more conservative.

Second, we examine the power of the $\pmis$-based hypothesis test to identify truly separated clusters. 
Specifically, we draw samples from two distinct Gaussian distributions, $N(\mu_1, \sigma^2)$ and $N(\mu_2, \sigma^2)$, where a range of $|\mu_1 - \mu_2|/\sigma$ values are selected from the set $\{1, 1.5, \dots, 7\}$. 
In each alternative scenario represented by a unique $|\mu_1 - \mu_2|/\sigma$ value, we draw 75 observations from each cluster distribution and compute the $p$-values using $\pmis$ as well as Gao et al.'s procedure. 
We calculate the power based on 500 replicates for each alternative scenario. 
As can be seen in the bottom panel of Figure \ref{fig:hyp_test}, $\pmis$ exhibits higher power than Gao et al.'s procedure for moderate-to-large degrees of cluster separation (defined by $|\mu_1 - \mu_2|/\sigma \in [2, 6]$).

\subsection{Real Data Applications} \label{sec:real_data}

\subsubsection{Cluster Analysis of Palmer Penguin Data}

We analyze the \texttt{penguins} data from the \texttt{palmerpenguins} package \citep{palmerpenguins}, which consists of bill, flipper, and mass measurements from three species of penguins in the Palmer Archipelago between 2007 and 2009. 
Following the analysis by \cite{gao2022selective}, we consider the subset of female penguins with complete data for bill and flipper length (both measured in millimeters), leaving us with 165 observations. Prior to performing the clustering, we center and scale the observations so that the measurements have zero mean and unit variance.
Following the procedure laid out in Section \ref{sec:kmeans_procedure}, we perform $k$-means clustering of the scaled observations into $K = 1, \dots, 8$ clusters and compute the gap statistic as well as $\pmis$ for each grouping of the observations. 
Figure \ref{fig:penguins_gap} shows the data used for clustering and the values of the gap statistic and $\pmis$ for different values of $K$.

For the $\pmis$ threshold of 0.05, the combined loss defined by the gap statistic and $\pmis$ is optimized at $K=3$.  
Both the visualization of the clustering data and the external species information suggest that the result is reasonable. We observe that the values of $\pmis$ in this example are strongly negatively correlated with other cluster validity indices \cite{rousseeuw1987silhouettes,lange2004stability,tibshirani2005cluster} to evaluate the $k$-means clustering partitions (Supplementary Table 2).
We repeat this analysis using hierarchical clustering based on the squared Euclidean distance with the Ward linkage and come to the same conclusions (Supplementary Figure 3 and Supplementary Table 3).

\begin{figure}[t!]
    \centering
    \includegraphics[]{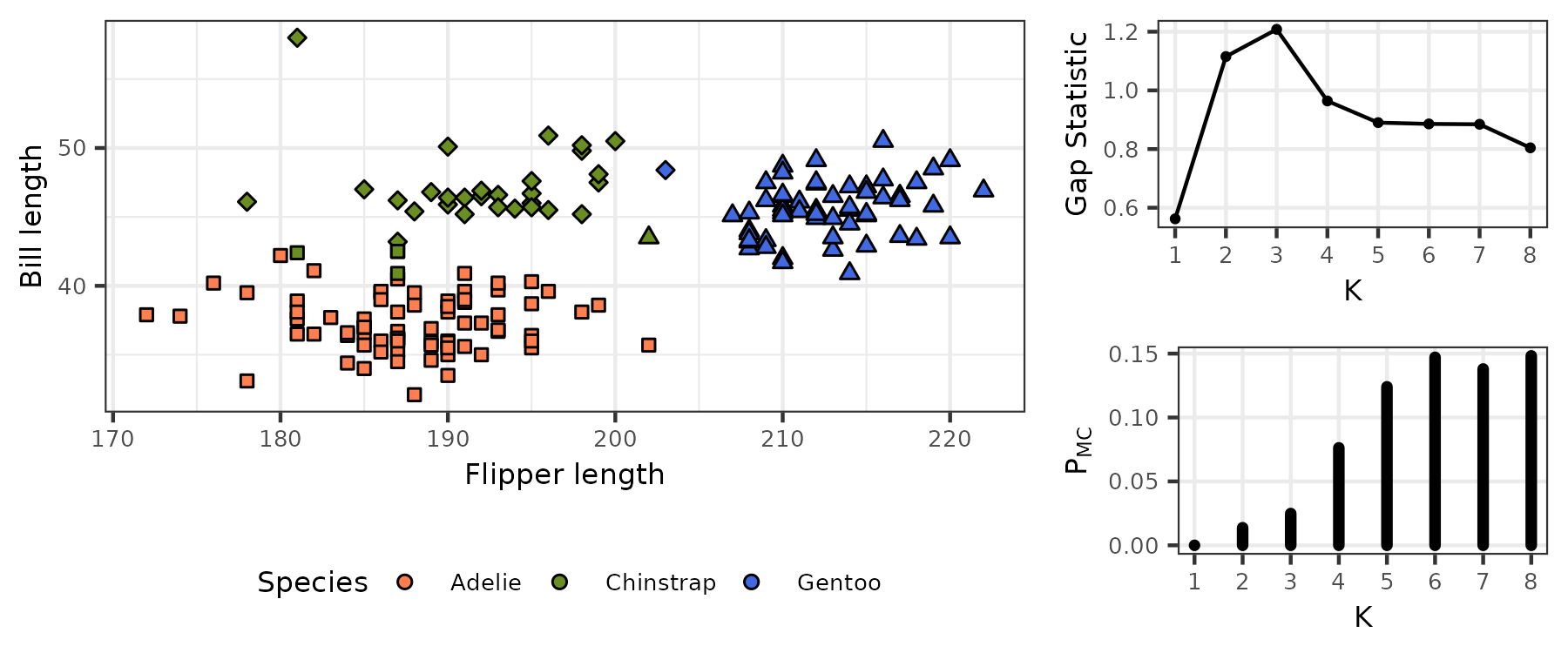}
    \caption{\textit{(Left)} Bill and flipper lengths for the subset of female penguins with complete data from the un-scaled \texttt{palmerpenguins} data. Color indicates species, while shape indicates the assigned $k$-means cluster at $K=3$. \textit{(Right)} Value of the gap statistic and $\pmis$ for different numbers of clusters based on the $k$-means clustering partition with Gaussian cluster distributions.}
    \label{fig:penguins_gap}
\end{figure}

\subsubsection{Inferring Population Structure from HGDP data} 

In this illustration, we apply the PHM algorithm to perform cluster analysis on the genetic data from the Human Genome Diversity Project (HGDP) \citep{cavalli2005human, conrad2006worldwide}, aiming to identify population structures. 
The dataset comprises 927 unrelated individuals sampled worldwide and genotyped at 2,543 autosomal SNPs. 
The geographic sampling locations are broadly divided into 7 continental groups: Europe, Central/South Asia (C/S Asia), Africa, Middle East, the Americas, East Asia, and Oceania.

Following the standard procedures for genetic data analysis, we pre-process the genotype matrix using principal component analysis (PCA) and select the first 5 PCs for our analysis based on the elbow point of the scree plot (Supplementary Figure 4). 
To apply the PHM algorithm, we fit a GMM to the dimension-reduced PC score matrix and select the model with 9 components based on the BIC. 
Each individual's posterior component assignment probability $\pi_k$ is shown in the distruct plot \cite{rosenberg2004distruct} in Figure \ref{fig:hgdp_plot}. 
Except for the Europe, Central/South Asia, and Middle East groups, the remaining continental groups tend to correspond to unique mixture components.

\begin{figure}[t!]
    \centering
    \includegraphics[]{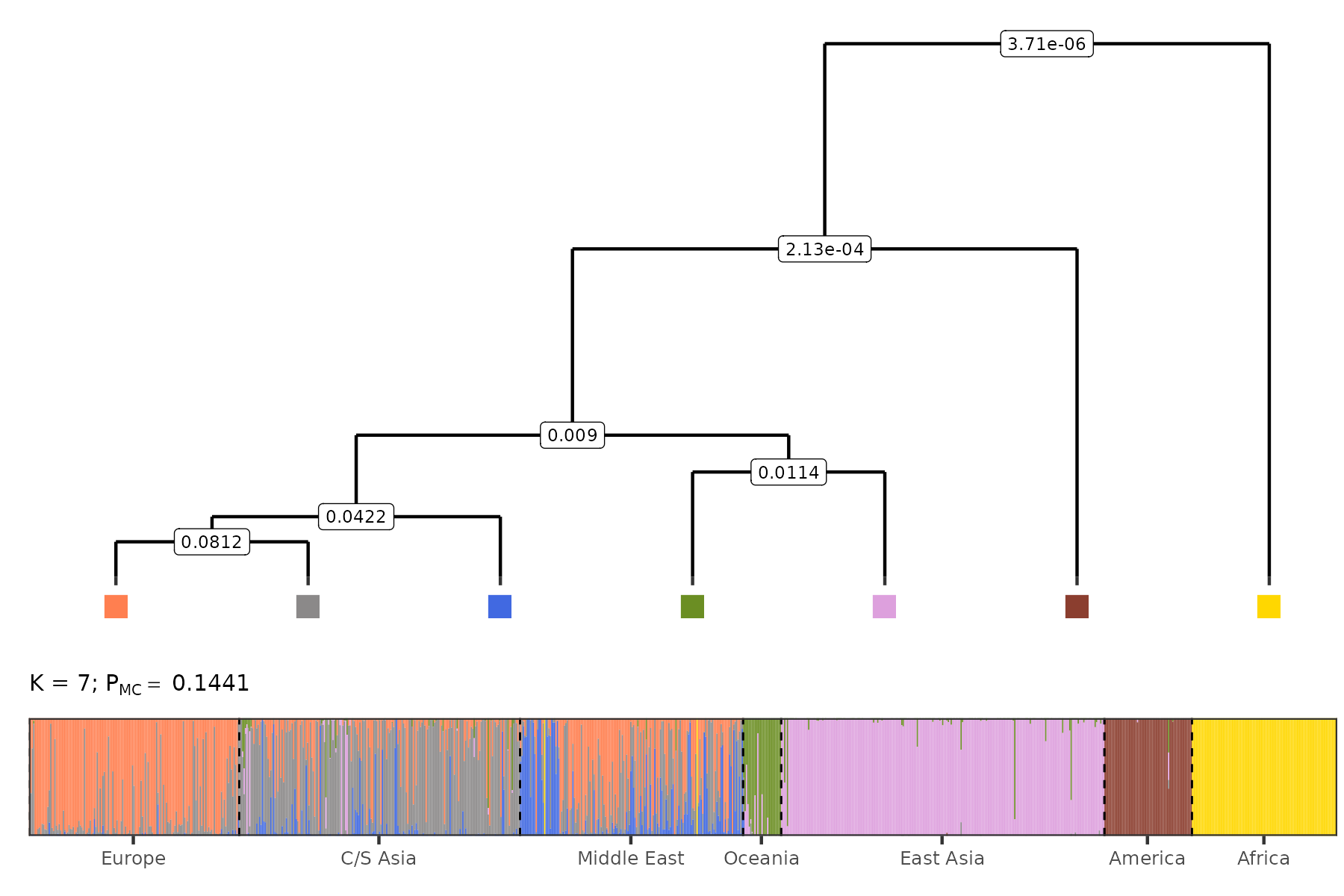}
    \caption{(\textit{Top}) Dendrogram visualizing the order of component merges in the PHM procedure. Numeric values indicate the reduction of $\pmis$ by the merge. Colors correspond to the mixture components from the distruct plot. (\textit{Bottom}) Distruct plot for posterior component assignment probability, with color indicating mixture components. Observations are grouped by the geographic region of their sampling location.}
    \label{fig:hgdp_plot}
\end{figure}

Starting with assigning each mixture component to its own cluster and an initial $\pmis = 0.1437$, we proceed with the steps of the PHM algorithm with $\tau = 0$, i.e., merging until all observations belong to a single cluster and $\pmis$ is decreased to 0.
Figure \ref{fig:hgdp_plot} visualizes the merging process as a dendrogram. The numbers in each node indicate the reduction of $\pmis$, i.e., $\Delta \pmis$, by combining the corresponding branches into a single cluster. 
A smaller value of $\Delta \pmis$ indicates that the clusters being combined are more distinct. 
Because the algorithm always prioritizes merging the most overlapping clusters, the merging sequence reflects the relative genetic dissimilarities between different clusters. This relationship can be straightforwardly interpreted from the dendrogram.
For example, the first two merges --- reducing $\pmis$ to 0.063 and 0.021, respectively --- form a cluster representing samples from Europe, the Middle East, and Central/South Asia, reflecting a close genetic relationship and noticeable genetic admixture among these population groups.
The general structure of the dendrogram can be roughly explained by the likely path of historical human migrations: from Africa into the Middle East, from the Middle East to Europe and Central/South Asia, from Central/South Asia to East Asia, and from East Asia to Oceania and the Americas. 
The overall pattern of human genetic diversity among the continental groups identified from our analysis is also corroborated by more sophisticated genetic analysis using additional information (e.g., haplotype analysis) \citep{conrad2006worldwide} and high-coverage genome sequencing data \citep{bergstrom2020insights}.

\subsubsection{Cluster Analysis of Single-cell RNA Sequence Data}

In this illustration, we apply the PHM algorithm to single-cell RNA sequencing (scRNA-seq) data from a sample of peripheral blood mononuclear cells to identify the different cell types. 
The data (sequenced on the Illumina NextSeq 500 and freely available from 10x Genomics) consist of gene expression counts for 2,700 single cells at 13,714 genes.

The raw sequence data are quality-controlled and pre-processed using standard procedures implemented in the \texttt{Seurat} package \citep{seuratpackage}, leaving us with 2,638 cells. 
Principal component analysis is subsequently performed on the normalized and scaled expression counts for dimension reduction. The first 10 components are selected based on the elbow point of the scree plot (Supplementary Figure 5).
The resulting 2,638 $\times$ 10 data matrix is used for our cluster analysis. 
Additionally, each cell is annotated as one of the following cell types: Na\"{i}ve CD4+ T cells, Memory CD4+ T cells, CD8+ T cells, Natural Killer (NK) cells, B cells, CD14+ monocytes, FCGR3A+ monocytes, Dendritic cells (DC), and Platelets, using known biomarkers. 
The annotated cell type information is not used in our cluster analysis procedure.

We fit a GMM on the dimension-reduced PC matrix and select the model with 9 components based on the optimal BIC. The posterior component assignment probability $\pi_k$ for each cell is visualized in Figure \ref{fig:seurat_plot}. With the exception of the Na\"{i}ve and Memory CD4+ T cells, each cell type corresponds to a unique mixture component.

Starting with each mixture component as its own cluster and an initial $\pmis = 0.0703$, we perform the merging procedure until all components have been combined into a single cluster. The merging process is represented by the dendrogram in Figure \ref{fig:seurat_plot}. 
We note that the pattern shown in the dendrogram has a striking similarity to the known immune cell differentiation trajectories. 
For example, the first merge combines the two types of CD4+ T cells, reducing $\pmis$ from 0.070 to 0.042.
Furthermore, CD8+ T cells, CD4+ T cells, NK cells, and B cells---all derived from Lymphoid progenitor cells---are grouped together in the dendrogram and separated from the branch consisting of FCGR4A+, CD14+ monocytes, and Dendritic cells---all of which are derived from Myeloid progenitor cells.

\begin{figure}[t!]
    \centering
    \includegraphics[]{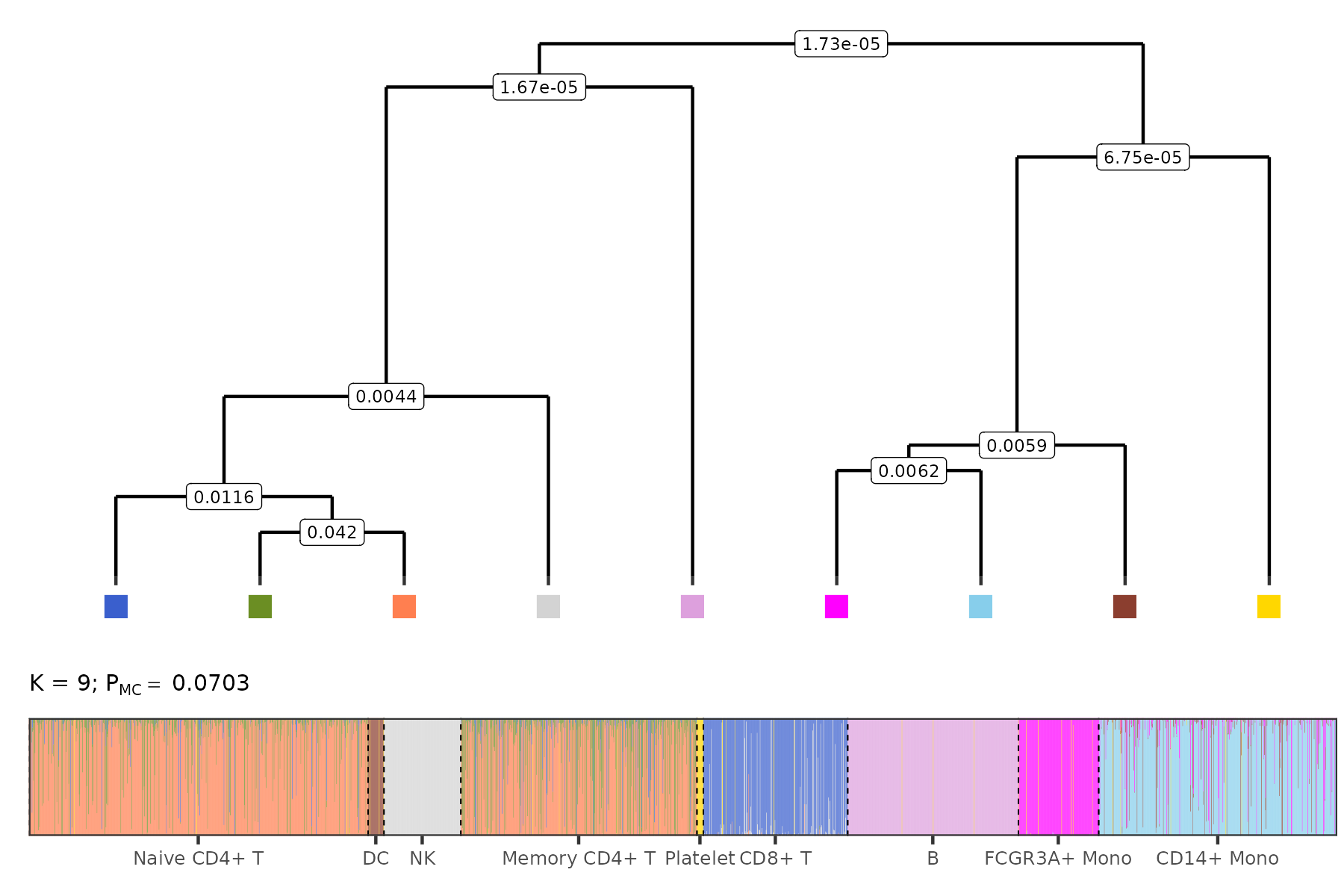}
    \caption{(\textit{Top}) Dendrogram visualizing the order of component merges in the PHM procedure. Numeric values indicate the reduction of $\pmis$ by the merge. Colors correspond to the mixture components from the distruct plot. (\textit{Bottom}) Distruct plot for posterior component assignment probability, with color indicating mixture components. Observations are grouped by annotated cell type.}
    \label{fig:seurat_plot}
\end{figure}

\section{Conclusion and Discussion}

In this work, we introduce the Distinguishability criterion, $\pmis$, to quantify the separability of clusters inferred from cluster analysis procedures.
We discuss the intuition behind the criterion, as well as the derivation and properties of $\pmis$.
We propose a combined loss function-based computational framework that integrates the Distinguishability criterion with available model and heuristics-based clustering algorithms and demonstrate its use with synthetic and real data applications.

The proposed $\pmis$ is an internal clustering validity index to assess the separability of the clustering results, with unique advantages over alternative validity indices. 
Since $\pmis$ is measured on the probability scale, the threshold for our proposed constrained optimization problem is interpretable. 
Additionally, $\pmis$ measurements are directly comparable across different datasets and various types of clustering applications,  enabling future work on assessing the replicability of clustering analysis.

While our numerical illustrations primarily use Gaussian or mixtures of Gaussian distributions to evaluate $\pmis$, it is important to highlight the proposed computational framework's flexibility and compatibility with any valid parametric likelihood function. 
As a result, the applications of $\pmis$ can be extended to a more diverse class of latent variable models, e.g., latent Dirichlet allocation (LDA) and generalized factor analysis models, in order to help address the similar model selection problems that arise with their use.
We will explore these extensions in our future work.

The properties of $\pmis$ are best utilized in our proposed PHM algorithm, which is motivated by Baudry et al.'s entropy criterion-based merging procedure \cite{Baudry_2010}.
By combining mixture components into clusters, both algorithms enjoy excellent model fit and interpretability for the inferred clusters. 
However, the constrained optimization formulation described above leads to a more interpretable stopping rule for the PHM algorithm compared to Baudry et al.'s procedure.
Furthermore, the PHM algorithm shows improved computational efficiency---due to the cluster merging property of $\pmis$, the complexity of the PHM algorithm is invariant to the sample size of the observed clustering data, making it more suitable for analyzing large-scale data.
Finally, we note that the dendrogram visualizing the complete merging procedure can be used to represent a coalescent process that has many applications across scientific disciplines, such as developmental biology, human genetics, and evolutionary biology. 
One possible future research direction is shifting the focus from cluster analysis to uncovering the underlying coalescent trees by incorporating more context-specific information, with the PHM algorithm as a natural starting point.

It is also possible to extend the PHM algorithm to work with a general class of hard clustering algorithms that output optimal partitions of observed data.
One strategy is to treat each output partition as a distinct population sample and estimate its distribution using a finite mixture model. We can then re-normalize the mixture proportions globally and initilize the clusters at the level of the output partitions, at which point applying the PHM algorithm is straightforward.
This simple strategy extends the applications of the Distinguishability criterion to a more diverse class of clustering algorithms, including density-based and graph clustering algorithms.

\section{Methods}

\subsection{Overview of Clustering Methods}

We briefly review existing clustering methods, focusing on the algorithms used in this paper. 
The available clustering methods can be roughly classified into two categories: heuristics-based clustering methods, represented by the $k$-means and hierarchical clustering algorithms, and model-based clustering methods, represented by finite-mixture models. 
The heuristics-based clustering methods typically do not make explicit distributional assumptions and instead perform hard clustering by outputting optimal partitions of the observed data based on their corresponding objective functions. 
The model-based clustering methods perform formal statistical inference of the latent cluster structures underlying the observed data.

\subsubsection{$k$-means clustering}

$k$-means clustering is a widely used heuristics-based clustering algorithm for partitioning a sample into clusters. 
As the name suggests, this procedure identifies clusters by their centroids (defined as the mean of all points in the cluster) and assigns observations to the cluster with the nearest centroid. More formally, for a specified number of clusters $K$, the algorithm groups the observations into disjoint sets $\mathcal{C}_1, \dots, \mathcal{C}_K$ to minimize the distortion function, i.e., 

\begin{equation*}
\min_{\mathcal{C}_1, \dots, \mathcal{C}_K} \sum_{k=1}^K \sum_{i \in \mathcal{C}_k} \bigg | \bigg | \xv_i - \sum_{j \in \mathcal{C}_k} \xv_j / |\mathcal{C}_k| \bigg | \bigg |_2^2,
\end{equation*}

which corresponds to minimizing the within-cluster variances. The clustering is usually performed using Lloyd's algorithm \cite{lloyd1982least}, which alternates between assigning observations to the cluster with the nearest centroid and updating the cluster centroid to eventually arrive at a locally optimal solution. 

It has been shown that the $k$-means algorithm is equivalent to an approximate maximum likelihood procedure, where the underlying probability model assumes a multivariate Gaussian distribution for each latent cluster \cite{Bock_1996,Fraley_2002}.
Hence, the observation frequency of each cluster and the corresponding parameters for the Gaussian distribution can be estimated straightforwardly using the partitioned output from the $k$-means algorithm. 

\subsubsection{Hierarchical clustering}

Hierarchical clustering produces a sequence of partitions for observations. In the commonly used agglomerative hierarchical clustering procedure, each observation comprises its own cluster in the initial partition. Subsequent clusterings in the sequence are formed by repeatedly combining the most similar pair of groups of observations into a single group until all observations have been combined into a single cluster.
The similarity between groups of observations in hierarchical clustering is typically defined by a distance measure between pairs of observations, such as the squared Euclidean distance, in addition to a function that generalizes this similarity to groups of observations (referred to as a linkage function). Commonly used linkage functions between groups of observations are single linkage, complete linkage, average linkage, and the Ward linkage.

The hierarchical clustering algorithm is also often connected to probability models assuming a Gaussian data distribution underlying each cluster \cite{Mclachlan_book, Fraley_2002}, and explicit Gaussian assumptions are commonly used for model-selection or post-selection inference in hierarchical clustering \cite{Kimes_2017, gao2022selective, Grabski_2023}.  

\subsubsection{Model-based Clustering}

Finite mixture models are the most representative approach for model-based clustering, where a mixture distribution with finite components models the observed clustering data.
Traditionally, each mixture component is taken to correspond to a homogeneous subpopulation or cluster. 
The goal of inference in the mixture model setting is to estimate the characteristics of each mixture component, i.e., $p(\xv \mid \theta = k)$, and the corresponding mixture proportion, i.e., $\alpha_k$.
The expectation-maximization (EM) algorithm is commonly used to find maximum likelihood estimates of the parameters of interest. 
The Gaussian mixture model (GMM), which models each mixture component using a unique Gaussian distribution, is probably the most commonly used mixture model in practice.
This is because, among other reasons, the GMM is considered to be a universal approximator \citep{Goodfellow_book} and is flexible enough to fit diverse data types.

Unlike most heuristics-based clustering approaches, model-based clustering algorithms perform soft (or fuzzy) clustering, as they do not directly partition the observed data. 
Instead, every data point has an associated (posterior) probability distribution over all possible cluster assignments.
A post-hoc classification procedure, with a pre-specified decision rule using the cluster assignment probabilities, can be applied to partition the observed data.

For a more thorough review of model-based clustering, we refer the reader to \cite{Fraley_2002, mcnicholas2016model, Gormley_2023}

\subsection{Derivation and Estimation of $\pmis$}  

The Bayes risk for a general classifier $\delta(\xv): \mathbb{R}^p \mapsto \{1, \dots, K\}$ assigning a point $\xv$ to one of the $K$ clusters can be derived as follows,

\begin{equation} \label{eq:pmis_form} 
    \begin{aligned}
    \pmis & = \mathbb{E}_{\xv} \Big[ 
        \mathbb{E}_\theta \big[ 
            L(\delta(\xv), \theta)
            \mid \xv
        \big] \Big]  \\ 
        & =  \int \left( \sum_{j=1}^K \sum_{i \neq j} \pi_i(\xv) \, \Pr( \delta(\xv) = j \mid \xv) \right)  P(d\xv) \\
        & =  \int \left( \sum_{j=1}^K  (1 - \pi_j(\xv)) \, \Pr( \delta(\xv) = j \mid \xv) \right)  P(d\xv)
    \end{aligned}
\end{equation}

For a more detailed derivation, see Appendix \ref{app:pmis_derivation}. The marginal data distribution $P(\xv)$, with loose notation, is given by

\begin{equation*}
P(\xv) = \sum_{k=1}^K \Pr(\theta=k) \, p(\xv \mid \theta = k) = \sum_{k=1}^K \alpha_k \, p(\xv \mid \theta = k). 
\end{equation*}
%Note that, for deterministic decision rules, $\Pr(\delta(\xv) = j \mid \xv) = \lv\{\delta(\xv) = j \}$.

For the default randomized decision rule, $\delta_r(\xv) \sim \mbox{Categorical}\, (\piv (\xv))$,

\begin{equation}
    \Pr(\delta(\xv) = j \mid \xv) = \pi_j(\xv).
\end{equation}

Thus,

\begin{equation} \label{eq:pmis_rand_defn}
\begin{aligned} 
    P_{{\rm mc}, \,\delta_r} &= \int \left( \sum_{j=1}^K  \pi_j(\xv) (1-\pi_j(\xv)) \right) P(d\xv) \\
     & = 2 \sum_{i<j} \int  \pi_i(\xv) \pi_j(\xv)  P(d\xv) 
\end{aligned}
\end{equation}

To compute $\pmis$ using the optimal decision rule, $\delta_o(\xv) = \arg\max_k \pi_k(\xv)$, we define a partition of the sample space, $\cup_{k=1}^K \mathcal{R}_k$, such that $\mathcal{R}_k := \{\xv: \delta_o(\xv) = k\}$. 
It follows that, 

\begin{equation}
\Pr(\delta(\xv) =j \mid \xv) = \boldsymbol{1}\{\xv \in \mathcal{R}_j\},
\end{equation}

and  

\begin{equation}
\begin{aligned}  \label{eq:pmis_optimal_defn}
    P_{{\rm mc}, \,\delta_o} & = \int \left( \sum_{j=1}^K \sum_{i \neq j}  \pi_i(\xv) \, \boldsymbol{1}\{\xv \in \mathcal{R}_j\} \right) P(d\xv) \\
                            & = \int \left(1 - \max_k \pi_k(\xv)\right) \, P(d \xv)
\end{aligned}
\end{equation}

Note that, 

\begin{equation}
    \sum_{j=1}^K  \pi_j(\xv) (1-\pi_j(\xv)) \ge \sum_{j=1}^k \pi_j(\xv) (1-\max_k \pi_k(\xv))  = 1 - \max_k \pi_k(\xv),~ \forall \xv
\end{equation}

Hence, 

\begin{equation*}
    P_{{\rm mc}, \,\delta_r}  \ge P_{{\rm mc}, \,\delta_o}
\end{equation*}

Unless otherwise specificied, we use the notation $\pmis$ to refer to $P_{{\rm mc}, \,\delta_r}$ by default.
In this paper, we estimate $\pmis$ by plugging in the point estimates of the $\alpha_k$'s as well as the key distributional parameters in the corresponding likelihood functions obtained from the observed data.

\subsection{Lower and Upper Bounds of $\pmis$}

When all clusters are well-separated, $\pmis$ approaches its lower bound at 0. 
More specifically, assuming well-separated clusters, both of the following conditions should hold:

 \begin{align*}
    & \pi_i(\xv)\, \pi_j(\xv) \to  0, ~ \forall \, \xv \mbox{ and } (i ,j) \mbox{ pairs}, \\
   \intertext{and} 
    & \max_k \pi_k(\xv) \to 1, ~\forall \, \xv.
\end{align*}

Hence, both decision rules ($\delta_r$ and $\delta_o$) approach perfect classification accuracy.

$\pmis$ is maximized when all clusters are completely overlapping, i.e., 

\begin{equation}
    p(\xv \mid \theta = i) = p(\xv \mid \theta = j)  ~ \forall \, \xv \mbox{ and } (i ,j) \mbox{ pairs}. 
\end{equation}

Thus, $\pi_k(x) = \alpha_k, ~ \forall \xv$.
It follows that 

\begin{equation*}
    \max P_{{\rm mc}, \,\delta_r} = \sum_{k=1}^K \alpha_k \, (1 - \alpha_k),
\end{equation*}

and 

\begin{equation*}
    \max P_{{\rm mc}, \,\delta_o} = 1 - \max_k \alpha_k.
\end{equation*}

In the special case that $\alpha_k = 1/K$ for all $k$ values,

\begin{equation}
    \max P_{{\rm mc}, \,\delta_r} = \max P_{{\rm mc}, \,\delta_o} = \frac{K-1}{K}.
\end{equation} 

\subsection{The Cluster Merging Property of $\pmis$}

The merging property is specific to the default $\pmis$, evaluated using the randomized decision rule $\delta_r$.
For a given cluster configuration with $K \ge 2$, consider merging two arbitrary clusters to form a new combined cluster.
Let $\pmis$ and $\pmis^\dagger$ denote the misclassification probabilities before and after the merging, respectively. 
The following proposition summarizes the merging property.

\begin{prop} \label{prop:pmis_merge}
   Merging two existing clusters indexed by $i$ and $j$ leads to  
   \begin{equation*} 
    \Delta \pmis^{(i,j)} := \pmis - \pmis^\dagger  = 2 \int \pi_i (\xv) \pi_j (\xv) \, P(d\xv) \ge 0.
   \end{equation*} 
   Furthermore, 
   \begin{equation*}
     \pmis = \sum_{i<j} \Delta \pmis^{(i,j)}
   \end{equation*}
\end{prop}

\begin{proof}
    Appendix \ref{app:merge_property}.
\end{proof}

The cluster merging property forms the basis of the PHM algorithm. As $\Delta \pmis^{(i,j)}$ can be pre-computed for all pairs of clusters from the initial configuration, the subsequent updates for $\pmis$---merging one pair of clusters at a time---become straightforward to compute.

\subsection{Numerical Evaluation of $\pmis$}

Evaluating $\pmis$ numerically can be challenging, especially when clustering data are high-dimensional.
For low-dimensional data, it is possible to evaluate Eqn (\ref{eq:pmis_form}) by numerical integration using various quadrature methods. 
However, they generally do not scale well when the clustering data dimensionality becomes larger than 5. 
When the marginal data distribution, $P(\xv)$, can be directly sampled from (as in the case in all examples presented in this paper), Monte Carlo (MC) integration becomes an efficient solution. 
Specifically, we sample $M$ data points from $P(\xv)$ and approximate $\pmis$ by

\begin{equation}
    \hat P_{{\rm mc}} = \frac{1}{M} \sum_{i=1}^{M} \left( \sum_{j=1}^K  \bigg(1 - \pi_j(\xv_i) \bigg) \, \Pr( \delta(\xv_i) = j \mid \xv) \right)
    \label{eq:pmis_montecarlo}
\end{equation}

The unique advantage of the Monte Carlo integration method is that its error bound is always $O(1/\sqrt{M})$ regardless of the dimensionality of $\xv$. 
We provide comparisons between the Monte Carlo method and the numerical integration for evaluating $\pmis$ in some low-dimensional settings (Supplementary Method 1 and Supplementary Table 4), indicating that the MC integration method is accurate and efficient. 

\clearpage

\bibliographystyle{unsrt}
\bibliography{references} 

\begin{thebibliography}{10}

\bibitem{braun2010identifying}
Elke Braun, Bart Geurten, and Martin Egelhaaf.
\newblock Identifying prototypical components in behaviour using clustering
  algorithms.
\newblock {\em PloS one}, 5(2):e9361, 2010.

\bibitem{wibisono2021multivariate}
S~Wibisono, MT~Anwar, Aji Supriyanto, and IHA Amin.
\newblock Multivariate weather anomaly detection using dbscan clustering
  algorithm.
\newblock In {\em Journal of Physics: Conference Series}, volume 1869, page
  012077. IOP Publishing, 2021.

\bibitem{ahmad2015techniques}
Parvez Ahmad, Saqib Qamar, and Syed Qasim~Afser Rizvi.
\newblock Techniques of data mining in healthcare: a review.
\newblock {\em International Journal of Computer Applications}, 120(15), 2015.

\bibitem{kao2017spatial}
Jui-Hung Kao, Ta-Chien Chan, Feipei Lai, Bo-Cheng Lin, Wei-Zen Sun, Kuan-Wu
  Chang, Fang-Yie Leu, and Jeng-Wei Lin.
\newblock Spatial analysis and data mining techniques for identifying risk
  factors of out-of-hospital cardiac arrest.
\newblock {\em International Journal of Information Management},
  37(1):1528--1538, 2017.

\bibitem{shafqat2020big}
Sarah Shafqat, Saira Kishwer, Raihan~Ur Rasool, Junaid Qadir, Tehmina Amjad,
  and Hafiz~Farooq Ahmad.
\newblock Big data analytics enhanced healthcare systems: a review.
\newblock {\em The Journal of Supercomputing}, 76:1754--1799, 2020.

\bibitem{xie2018qubic2}
Juan Xie, Anjun Ma, Yu~Zhang, Bingqiang Liu, Changlin Wan, Sha Cao, Chi Zhang,
  and Qin Ma.
\newblock Qubic2: a novel biclustering algorithm for large-scale bulk
  rna-sequencing and single-cell rna-sequencing data analysis.
\newblock {\em bioRxiv}, page 409961, 2018.

\bibitem{kiselev2019challenges}
Vladimir~Yu Kiselev, Tallulah~S Andrews, and Martin Hemberg.
\newblock Challenges in unsupervised clustering of single-cell rna-seq data.
\newblock {\em Nature Reviews Genetics}, 20(5):273--282, 2019.

\bibitem{kanter2019cluster}
Itamar Kanter, Piero Dalerba, and Tomer Kalisky.
\newblock A cluster robustness score for identifying cell subpopulations in
  single cell gene expression datasets from heterogeneous tissues and tumors.
\newblock {\em Bioinformatics}, 35(6):962--971, 2019.

\bibitem{qi2020clustering}
Ren Qi, Anjun Ma, Qin Ma, and Quan Zou.
\newblock Clustering and classification methods for single-cell rna-sequencing
  data.
\newblock {\em Briefings in bioinformatics}, 21(4):1196--1208, 2020.

\bibitem{wolfe1963object}
John~Harmon Wolfe.
\newblock {\em Object cluster analysis of social areas}.
\newblock PhD thesis, University of California, 1963.

\bibitem{cormack1971review}
Richard~M Cormack.
\newblock A review of classification.
\newblock {\em Journal of the Royal Statistical Society: Series A (General)},
  134(3):321--353, 1971.

\bibitem{oyewole2023data}
Gbeminiyi~John Oyewole and George~Alex Thopil.
\newblock Data clustering: application and trends.
\newblock {\em Artificial Intelligence Review}, 56(7):6439--6475, 2023.

\bibitem{Bock_1996}
Hans~H Bock.
\newblock Probabilistic models in cluster analysis.
\newblock {\em Computational Statistics \& Data Analysis}, 23(1):5--28, 1996.

\bibitem{Fraley_2002}
Chris Fraley and Adrian~E Raftery.
\newblock Model-based clustering, discriminant analysis, and density
  estimation.
\newblock {\em Journal of the American Statistical Association},
  97(458):611–631, June 2002.

\bibitem{Kimes_2017}
Patrick~K. Kimes, Yufeng Liu, David~Neil Hayes, and James~Stephen Marron.
\newblock Statistical significance for hierarchical clustering.
\newblock {\em Biometrics}, 73(3):811–821, January 2017.

\bibitem{gao2022selective}
Lucy~L Gao, Jacob Bien, and Daniela Witten.
\newblock Selective inference for hierarchical clustering.
\newblock {\em Journal of the American Statistical Association}, pages 1--11,
  2022.

\bibitem{chen2022selective}
Yiqun~T Chen and Daniela~M Witten.
\newblock Selective inference for k-means clustering.
\newblock {\em arXiv preprint arXiv:2203.15267}, 2022.

\bibitem{Grabski_2023}
Isabella~N. Grabski, Kelly Street, and Rafael~A. Irizarry.
\newblock Significance analysis for clustering with single-cell rna-sequencing
  data.
\newblock {\em Nature Methods}, 20(8):1196–1202, July 2023.

\bibitem{Hennig_2015}
Christian Hennig.
\newblock What are the true clusters?
\newblock {\em Pattern Recognition Letters}, 64:53–62, October 2015.

\bibitem{tibshirani2001estimating}
Robert Tibshirani, Guenther Walther, and Trevor Hastie.
\newblock Estimating the number of clusters in a data set via the gap
  statistic.
\newblock {\em Journal of the Royal Statistical Society: Series B (Statistical
  Methodology)}, 63(2):411--423, 2001.

\bibitem{Halkidi_2001}
Maria Halkidi, Yannis Batistakis, and Michalis Vazirgiannis.
\newblock On clustering validation techniques.
\newblock {\em Journal of intelligent information systems}, 17:107--145, 2001.

\bibitem{Kim_2005}
Minho Kim and RS~Ramakrishna.
\newblock New indices for cluster validity assessment.
\newblock {\em Pattern Recognition Letters}, 26(15):2353--2363, 2005.

\bibitem{Liu_2010}
Yanchi Liu, Zhongmou Li, Hui Xiong, Xuedong Gao, and Junjie Wu.
\newblock Understanding of internal clustering validation measures.
\newblock In {\em 2010 IEEE international conference on data mining}, pages
  911--916. IEEE, 2010.

\bibitem{rousseeuw1987silhouettes}
Peter~J Rousseeuw.
\newblock Silhouettes: a graphical aid to the interpretation and validation of
  cluster analysis.
\newblock {\em Journal of computational and applied mathematics}, 20:53--65,
  1987.

\bibitem{calinski1974dendrite}
Tadeusz Cali{\'n}ski and Jerzy Harabasz.
\newblock A dendrite method for cluster analysis.
\newblock {\em Communications in Statistics-theory and Methods}, 3(1):1--27,
  1974.

\bibitem{dunn1974well}
Joseph~C Dunn.
\newblock Well-separated clusters and optimal fuzzy partitions.
\newblock {\em Journal of cybernetics}, 4(1):95--104, 1974.

\bibitem{Melnykov_2016}
Volodymyr Melnykov.
\newblock Merging mixture components for clustering through pairwise overlap.
\newblock {\em Journal of Computational and Graphical Statistics},
  25(1):66–90, January 2016.

\bibitem{Celeux_1996}
Gilles Celeux and Gilda Soromenho.
\newblock An entropy criterion for assessing the number of clusters in a
  mixture model.
\newblock {\em Journal of Classification}, 13(2):195–212, September 1996.

\bibitem{Biernacki_2000}
C.~Biernacki, G.~Celeux, and G.~Govaert.
\newblock Assessing a mixture model for clustering with the integrated
  completed likelihood.
\newblock {\em IEEE Transactions on Pattern Analysis and Machine Intelligence},
  22(7):719–725, July 2000.

\bibitem{Baudry_2010}
Jean-Patrick Baudry, Adrian~E. Raftery, Gilles Celeux, Kenneth Lo, and Raphaël
  Gottardo.
\newblock Combining mixture components for clustering.
\newblock {\em Journal of Computational and Graphical Statistics},
  19(2):332–353, January 2010.

\bibitem{Luxburg_2010}
Ulrike {Von Luxburg} et~al.
\newblock Clustering stability: an overview.
\newblock {\em Foundations and Trends{\textregistered} in Machine Learning},
  2(3):235--274, 2010.

\bibitem{lange2004stability}
Tilman Lange, Volker Roth, Mikio~L Braun, and Joachim~M Buhmann.
\newblock Stability-based validation of clustering solutions.
\newblock {\em Neural computation}, 16(6):1299--1323, 2004.

\bibitem{tibshirani2005cluster}
Robert Tibshirani and Guenther Walther.
\newblock Cluster validation by prediction strength.
\newblock {\em Journal of Computational and Graphical Statistics},
  14(3):511--528, 2005.

\bibitem{Hennig_2010}
Christian Hennig.
\newblock Methods for merging gaussian mixture components.
\newblock {\em Advances in Data Analysis and Classification}, 4(1):3–34,
  January 2010.

\bibitem{mclust}
Luca Scrucca, Michael Fop, T.~Brendan Murphy, and Adrian~E. Raftery.
\newblock {mclust} 5: clustering, classification and density estimation using
  {G}aussian finite mixture models.
\newblock {\em The {R} Journal}, 8(1):289--317, 2016.

\bibitem{ClusterR}
Lampros Mouselimis.
\newblock {\em {ClusterR}: Gaussian Mixture Models, K-Means, Mini-Batch-Kmeans,
  K-Medoids and Affinity Propagation Clustering}, 2023.
\newblock R package version 1.3.1.

\bibitem{arthur2007k}
David Arthur, Sergei Vassilvitskii, et~al.
\newblock k-means++: The advantages of careful seeding.
\newblock In {\em Soda}, volume~7, pages 1027--1035, 2007.

\bibitem{hinkley1988bootstrap}
David~V Hinkley.
\newblock Bootstrap methods.
\newblock {\em Journal of the Royal Statistical Society Series B: Statistical
  Methodology}, 50(3):321--337, 1988.

\bibitem{palmerpenguins}
Allison~Marie Horst, Alison~Presmanes Hill, and Kristen~B Gorman.
\newblock {\em palmerpenguins: Palmer Archipelago (Antarctica) penguin data},
  2020.
\newblock R package version 0.1.0.

\bibitem{cavalli2005human}
L~Luca Cavalli-Sforza.
\newblock The human genome diversity project: past, present and future.
\newblock {\em Nature Reviews Genetics}, 6(4):333--340, 2005.

\bibitem{conrad2006worldwide}
Donald~F Conrad, Mattias Jakobsson, Graham Coop, Xiaoquan Wen, Jeffrey~D Wall,
  Noah~A Rosenberg, and Jonathan~K Pritchard.
\newblock A worldwide survey of haplotype variation and linkage disequilibrium
  in the human genome.
\newblock {\em Nature genetics}, 38(11):1251--1260, 2006.

\bibitem{rosenberg2004distruct}
Noah~A Rosenberg.
\newblock Distruct: a program for the graphical display of population
  structure.
\newblock {\em Molecular ecology notes}, 4(1):137--138, 2004.

\bibitem{bergstrom2020insights}
Anders Bergstr{\"o}m, Shane~A McCarthy, Ruoyun Hui, Mohamed~A Almarri, Qasim
  Ayub, Petr Danecek, Yuan Chen, Sabine Felkel, Pille Hallast, Jack Kamm,
  et~al.
\newblock Insights into human genetic variation and population history from 929
  diverse genomes.
\newblock {\em Science}, 367(6484):eaay5012, 2020.

\bibitem{seuratpackage}
Yuhan Hao, Tim Stuart, Madeline~H Kowalski, Saket Choudhary, Paul Hoffman,
  Austin Hartman, Avi Srivastava, Gesmira Molla, Shaista Madad, Carlos
  Fernandez-Granda, and Rahul Satija.
\newblock Dictionary learning for integrative, multimodal and scalable
  single-cell analysis.
\newblock {\em Nature Biotechnology}, 2023.

\bibitem{lloyd1982least}
Stuart Lloyd.
\newblock Least squares quantization in pcm.
\newblock {\em IEEE transactions on information theory}, 28(2):129--137, 1982.

\bibitem{Mclachlan_book}
Geoffrey~J McLachlan, Sharon~X Lee, and Suren~I Rathnayake.
\newblock Finite mixture models.
\newblock {\em Annual review of statistics and its application}, 6:355--378,
  2019.

\bibitem{Goodfellow_book}
Ian Goodfellow, Yoshua Bengio, and Aaron Courville.
\newblock {\em Deep learning}.
\newblock MIT press, 2016.

\bibitem{mcnicholas2016model}
Paul~D McNicholas.
\newblock Model-based clustering.
\newblock {\em Journal of Classification}, 33:331--373, 2016.

\bibitem{Gormley_2023}
Isobel~Claire Gormley, Thomas~Brendan Murphy, and Adrian~E Raftery.
\newblock Model-based clustering.
\newblock {\em Annual Review of Statistics and Its Application}, 10:573--595,
  2023.

\end{thebibliography}

\clearpage

\begin{appendices}
    
    \section{Derivation of $\pmis$}  \label{app:pmis_derivation}
            
    To compute the Bayes risk for a general classifier $\delta(\xv)$ under the 0-1 loss, we first evaluate its posterior expected loss $\mathbb{E}_\theta \big[ L(\delta(\xv), \theta) \mid \xv \big]$ as follows: 
    \begin{align*}
            \mathbb{E}_\theta \big[ L(\delta(\xv), \theta) \mid \xv \big] & = \Pr( \theta \ne \delta(\xv) \mid \xv) \\
                    & = \sum_{j=1}^K \Pr(\delta(\xv) =j, \, \theta \ne j \mid  \xv) \\
                    & = \sum_{j=1}^K \Pr( \theta \ne j  \mid \,  \xv) \, \Pr(\delta(\xv) = j \mid \xv) \\
                    & =  \sum_{j=1}^K \sum_{i \ne j} \Pr(\theta = i \mid \, \xv ) \, \Pr(\delta(\xv) = j \mid \xv)  \\
                    & = \sum_{j=1}^k \sum_{i \ne j} \pi_i(\xv)\,   \Pr(\delta(\xv) = j \mid \xv) \\
                    & =  \sum_{j=1}^k  (1-\pi_j(\xv)) \,   \Pr(\delta(\xv) = j \mid \xv) 
    \end{align*}
    
    Note that $\Pr(\theta \ne j \mid \xv, \delta(\xv)) = \Pr(\theta \ne j \mid \xv)$.
    Subsequently,
    \begin{equation}
        \begin{aligned} 
                \pmis & = \mathbb{E}_{\xv} \Big[ 
                    \mathbb{E}_\theta \big[ 
                        L(\delta(\xv), \theta)
                        \mid \xv
                    \big] \Big] \nonumber \\
               %     & = \mathbb{E}_{\xv} \left[ \sum_{j=1}^k \sum_{i \ne j} \pi_i(\xv)\,   \Pr(\delta(\xv) = j \mid \xv) \right] \nonumber \\
                    & =  \int \left( \sum_{j=1}^K \sum_{i \neq j} \pi_i(\xv) \, \Pr( \delta(\xv) = j \mid \xv) \right)  P(d\xv) \\
                    & =  \int \left( \sum_{j=1}^K  (1 - \pi_j(\xv)) \, \Pr( \delta(\xv) = j \mid \xv) \right)  P(d\xv) 
            \end{aligned}
    \end{equation}
            
    \section{Proof of Proposition 1} \label{app:merge_property}
    
    \begin{proof}
    Consider merging two existing clusters $i,j$ to a new combined cluster $k'$. It follows that

    \begin{align*}
        & \alpha_{k'} = \alpha_{i} + \alpha_{j} \\
        \intertext{and}
        & p( \xv \mid \theta=k) = \frac{\alpha_i \, p( \xv \mid \theta=i) + \alpha_j \, p( \xv \mid \theta=j)} {\alpha_{k'}}   
    \end{align*}
    
    Consequently, by applying Bayes rule,
    
    \begin{equation} \label{eqn.comb_pi}
        \pi_{k'}(\xv) = \pi_{i}(\xv) + \pi_{j}(\xv)
    \end{equation}
            
    Let $S$ denote the set of indices of the existing clusters not impacted by the merge, where $|S| = K-2$. By Eqn (\ref{eq:pmis_rand_defn}), $\pmis$ can be written as 

    \begin{equation*}
        \begin{aligned}
            \pmis & = 2 \sum_{m,n \in S, \, m < n} \int \pi_m (\xv) \pi_n (\xv) P(d\xv)  \\
                & + 2 \sum_{l \in S} \int \pi_l (\xv) \pi_i (\xv) P(d\xv) +  2 \sum_{l \in S} \int \pi_l (\xv) \pi_j (\xv) P(d\xv) \\
                & + 2 \int \pi_i (\xv) \pi_j (\xv) P(d\xv)  
        \end{aligned}
    \end{equation*} 

    By Eqn (\ref{eqn.comb_pi}),
    \begin{equation*}
        2 \sum_{l \in S} \int \pi_l (\xv) \pi_i (\xv) P(d\xv) +  2 \sum_{l \in S} \int \pi_l (\xv) \pi_j (\xv) P(d\xv) = 2 \sum_{l \in S} \int \pi_l(\xv) \pi_{k'} (\xv)  P(d\xv) 
    \end{equation*}

    and note that, 

    \begin{equation*}
        \pmis^\dagger = 2 \sum_{m,n \in S, \, m < n} \int \pi_m (\xv) \pi_n (\xv) P(d\xv)  +  2 \sum_{l \in S} \int \pi_l(\xv) \pi_{k'} (\xv)  P(d\xv) 
    \end{equation*}

    It becomes evident that 
        \begin{equation}
            \Delta \pmis^{(i,j)} = \pmis - \pmis^\dagger = 2 \int \pi_i (\xv) \pi_j (\xv) P(d\xv) \ge 0 
        \end{equation}

        Plugging in the expression of $\Delta \pmis^{(i,j)}$ into Eqn  (\ref{eq:pmis_rand_defn}) yields
        \begin{equation}
            \pmis = \sum_{i<j} \Delta \pmis^{(i,j)} 
        \end{equation}            
    \end{proof}
            
    \paragraph{Remark} The merging property is specific to the randomized decision rule $\delta_r$ under the 0-1 loss. For the optimal decision rule, $\delta_o$, it can be shown that $\pmis^\dagger \le \pmis$ after merging a pair of existing clusters. 
    However, the quantitive expression for $\Delta \pmis^{(i,j)}$ is analytically intractable. 
            
    \section{$\pmis$ Dendrogram Construction}\label{app:dendrogram}
            
    We use a dendrogram to visualize the PHM procedure described in Algorithm \ref{alg:pmis_gmm_algorithm}. The leaf nodes in the tree correspond to the individual components of the mixture model (MM) used to fit the data.
    The edges between parent and child nodes correspond to a cluster merge step; we present the $\Delta \pmis$ reduction from the merge on the parent node for the merged clusters.
    In this way, it is possible to determine the value of $\pmis$ after each successive merge by subtracting the cumulative $\Delta \pmis$ values from the leaf nodes of the tree up to the height of a specific merge from the $\pmis$ of the initial cluster configuration.
            
    The height of a merge in the tree is determined as follows. 
    Let $\pmis^0$ denote the $\pmis$ value at the initial cluster configuration and, for a given merge, let $\pmis^{\ddagger}$ be the value of the criterion \textit{prior} to that merge taking place.
    We place the merge in the dendrogram at a height corresponding to the $\log_{10}$ scaled ratio of these values, i.e., $\log_{10} \pmis^0 / \pmis^{\ddagger}$.
    The $\log_{10}$ transformation prevents the merges from early on in the procedure from being ``squashed" to the bottom of the tree.
    We opt to use the $\pmis$ value before the merge occurs rather than after to avoid dividing by zero when calculating the height of the final merge (which would result in $\pmis = 0$). 

    \end{appendices}

\clearpage

\setcounter{equation}{0}
\setcounter{figure}{0}
\setcounter{table}{0}
\setcounter{page}{1}
\setcounter{section}{0}

\newcommand{\header}[1]{%
    \begin{center}
    \LARGE{\textbf{#1}}
    \end{center}
}

% Set figure labels
\renewcommand{\figurename}{Supplementary Figure}
\renewcommand{\tablename}{Supplementary Table}

\header{Supplementary Figures}

\newpage

\begin{figure}[ht!]
    \centering
    \includegraphics[]{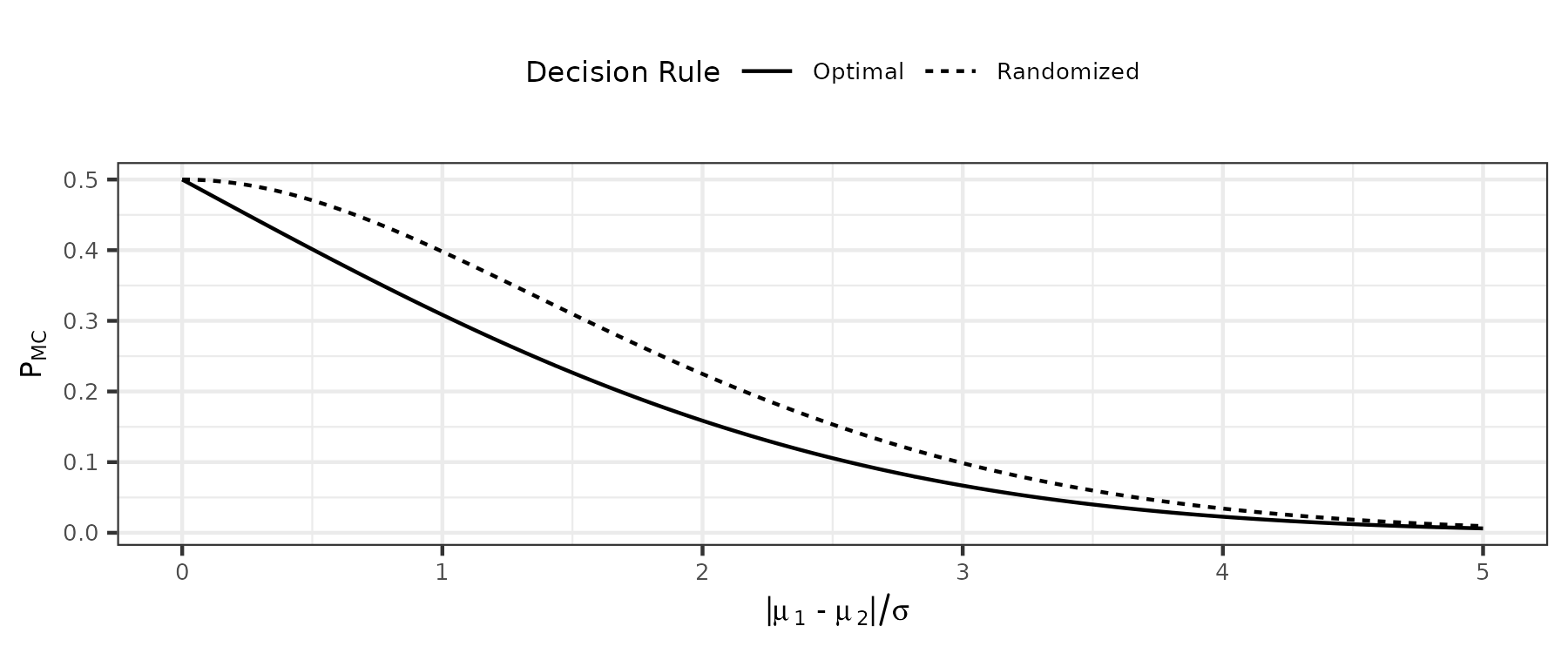}
    \caption{Values of $\pmis$ based on the randomized and optimal decision rules $\delta_r$ and $\delta_o$. The value $\pmis$ is shown in the y-axis and is calculated for two univariate Gaussian distributions $N(\mu_1, \sigma)$ and $N(\mu_2, \sigma)$ where $\pi_1 = \pi_2 = 0.5$. The x-axis indicates the degree of cluster separation in terms of the distribution parameters.}
    \label{decison_rules}
\end{figure}

\newpage

\begin{figure}[ht!]
    \centering
    \includegraphics[]{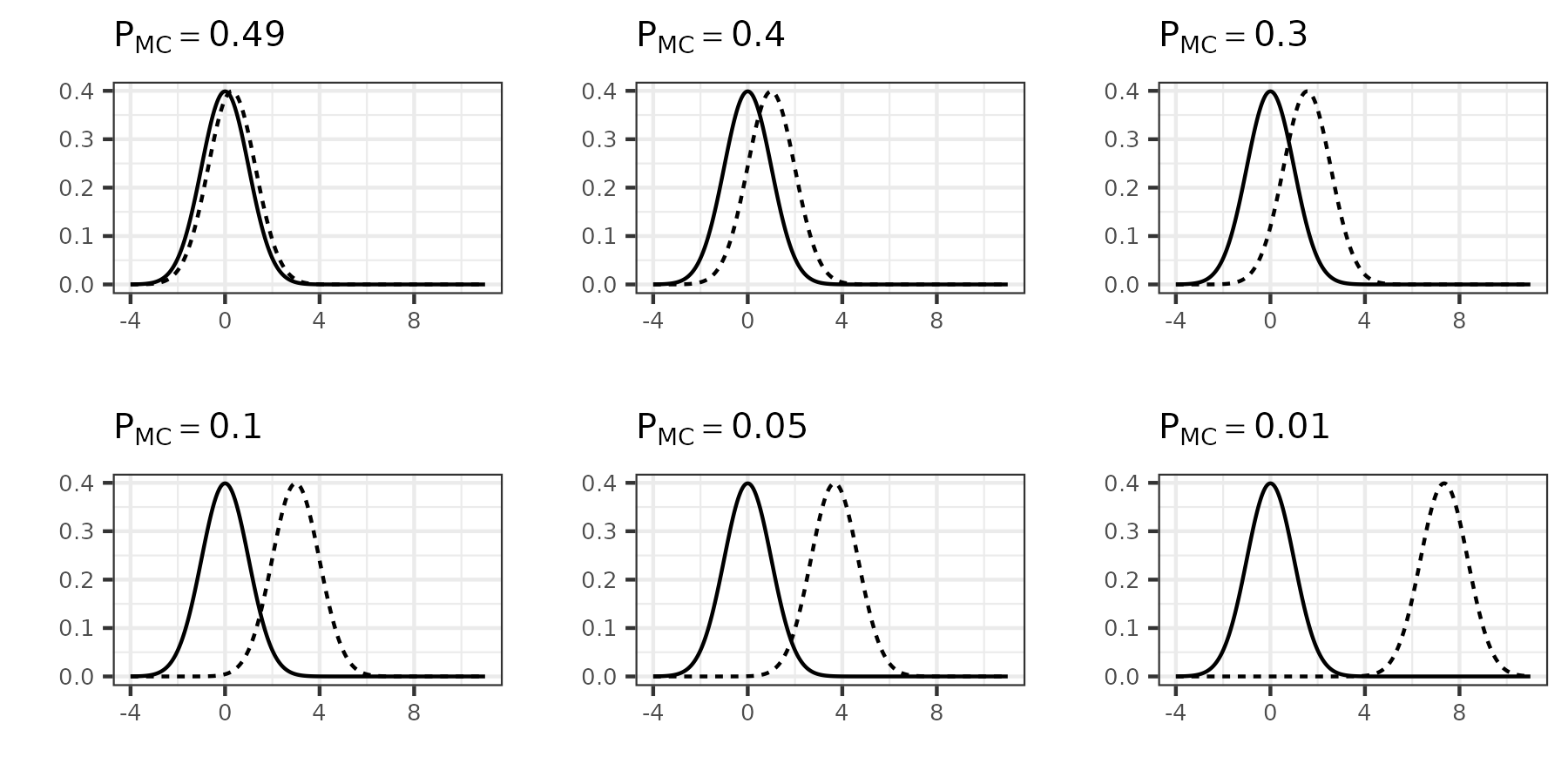}
    \caption{Distribution plots for two univariate Gaussian distributions $N(0, 1)$ (solid line) and $N(\mu, 1)$ (dashed line) at decreasing values of $\pmis$, where $\pi_1 = \pi_2 = 0.5$. The distance between the two centroids, $|\mu|$, determines the specific $\pmis$ value.}
    \label{pmis_example_distributions}
\end{figure}

\newpage

\begin{figure}[ht!]
    \centering
    \includegraphics[]{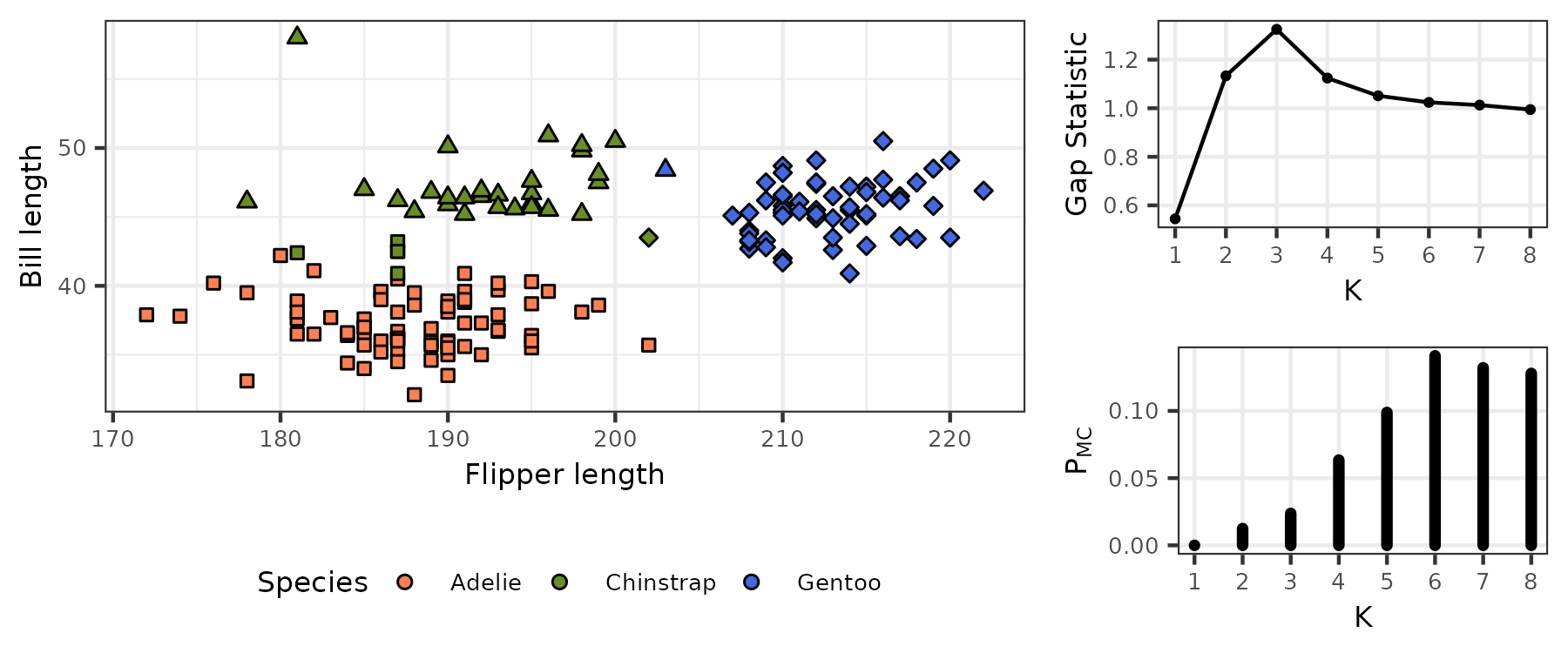}
    \caption{(\textit{Left}) Bill and flipper lengths for the subset of \texttt{palmerpenguins} data. Color indicates species while shape indicates the assigned hierarchical clustering partition. (\textit{Right}) Value of the gap statistic and $\pmis$ based on hierarchical clustering for different numbers of clusters with Gaussian cluster distributions.}
    \label{fig:penguins_hclust}
\end{figure}

\newpage

\begin{figure}[ht!]
    \centering
    \includegraphics[]{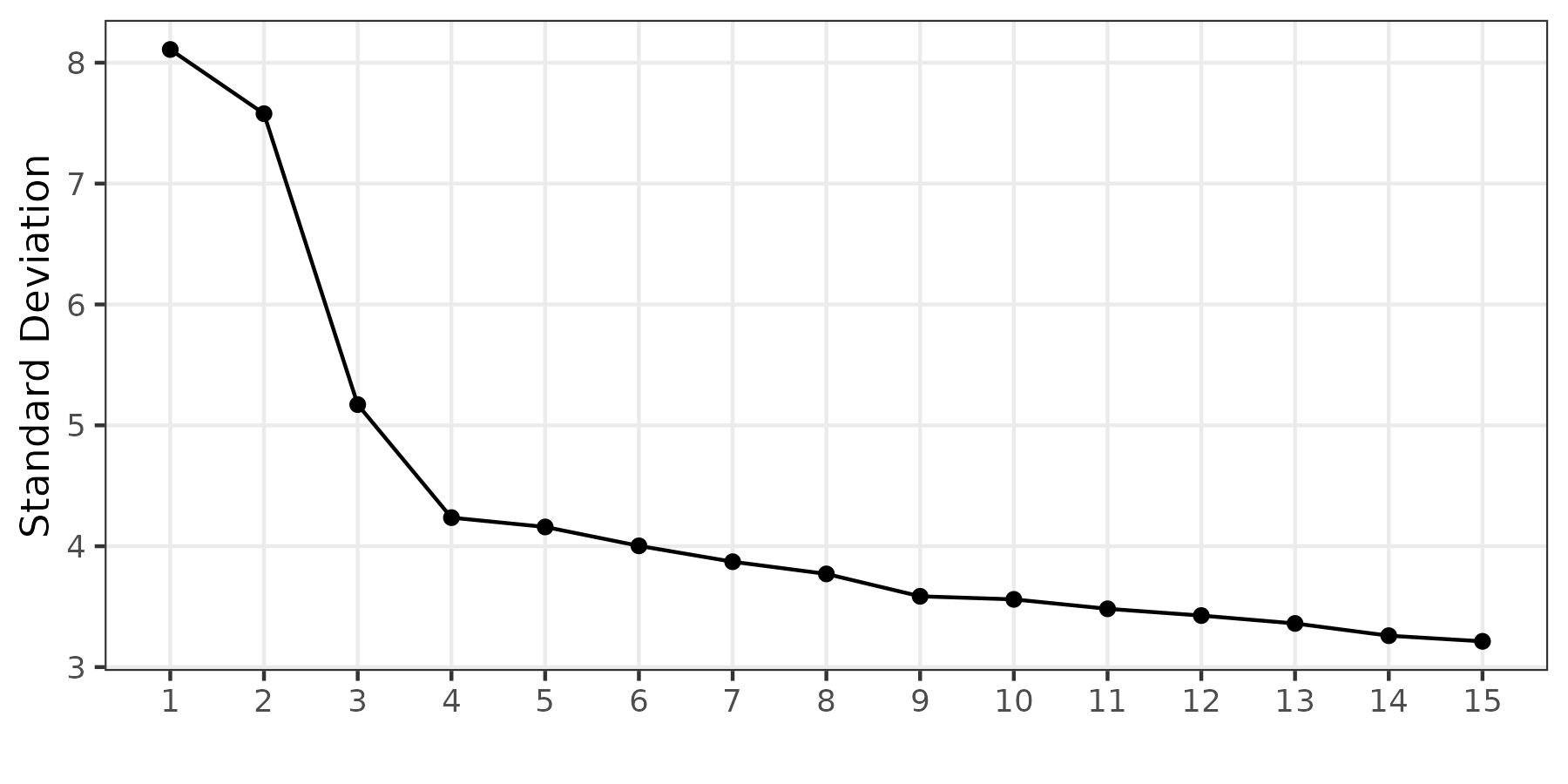}
    \caption{Scree plot visualizing standard deviation captured by each of of the principal component vectors from the HGDP data.}
    \label{fig:hgdp_screeplot}
\end{figure}

\newpage

\begin{figure}[ht!]
    \centering
    \includegraphics[]{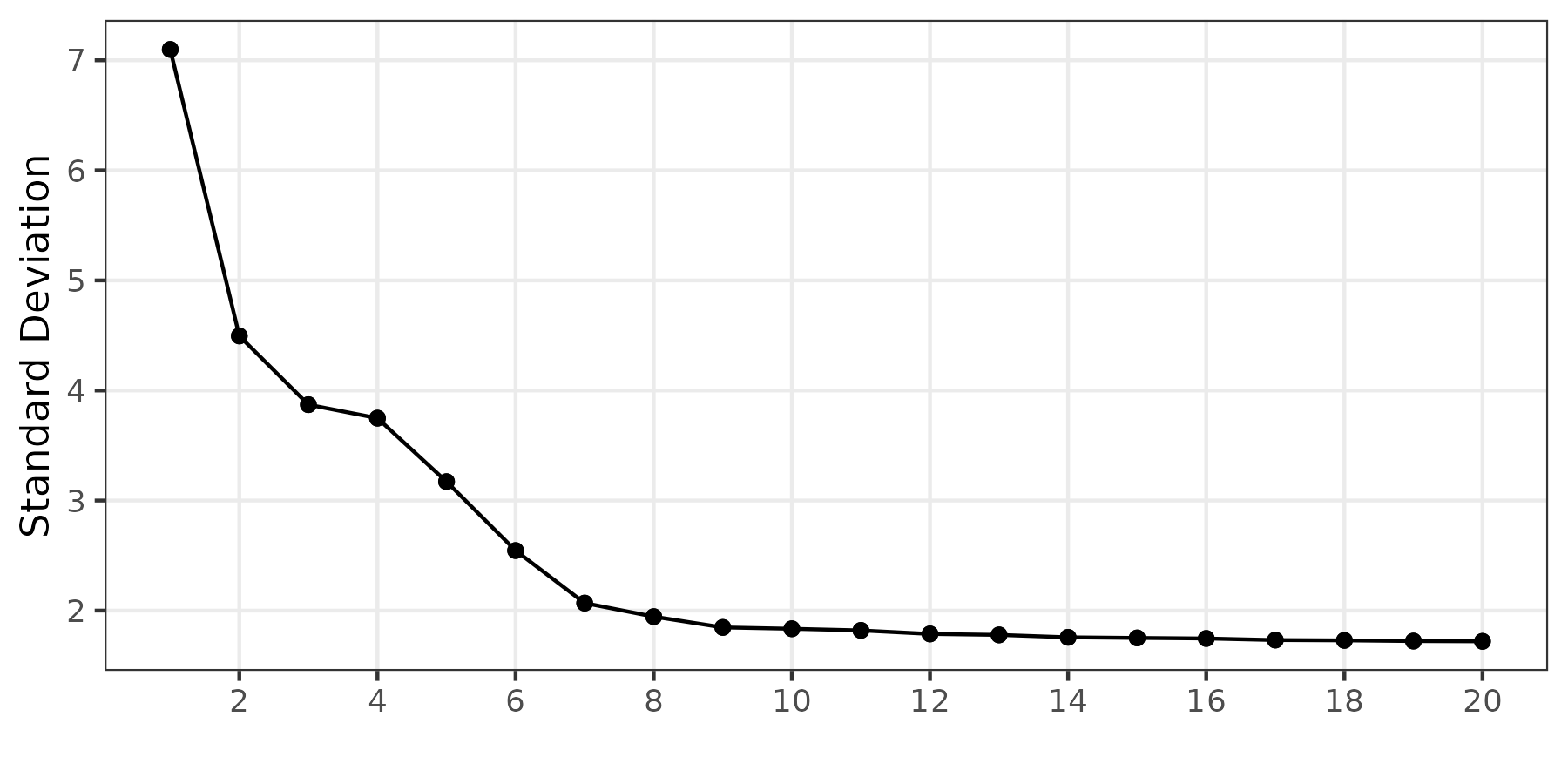}
    \caption{Scree plot visualizing standard deviation captured by each of of the principal component vectors from the scRNA-seq data.}
    \label{fig:pbmc_screeplot}
\end{figure}

\newpage

\header{Supplementary Tables}

\newpage

\begin{table}[ht!]
    \centering
    \begin{tabular}{rrrrrr}
      \toprule
        $K$ & Gap & $\pmis$ & Silhouette & Stability (ARI) & Prediction Strength \\ 
      \midrule
            1 & 0.403 & 0.000 &  ---    & ---   &   --- \\ 
            2 & 0.824 & 0.002 &  0.632  & 0.995 & 1.000 \\ 
            3 & 1.058 & 0.062 &  0.522  & 0.961 & 0.842 \\ 
            4 & 0.901 & 0.101 &  0.417  & 0.758 & 0.618 \\ 
            5 & 0.808 & 0.126 &  0.353  & 0.723 & 0.519 \\ 
            6 & 0.722 & 0.137 &  0.353  & 0.630 & 0.679 \\ 
            7 & 0.740 & 0.158 &  0.348  & 0.619 & 0.399 \\  
       \bottomrule
    \end{tabular}
    \caption{Values of $\pmis$ and other cluster validity indices for the simulated $k$-means example data.}
    \label{tab:overlap_simul_pmis_stability}
\end{table}

\newpage

\begin{table}[ht!]
    \centering
    \begin{tabular}{rrrrrr}
      \toprule
      $K$ & Gap & $\pmis$ & Silhouette & Stability (ARI) & Prediction Strength \\ 
      \midrule
      1 & 0.562 & 0.000 &   --- &   --- &   --- \\ 
      2 & 1.115 & 0.014 & 0.583 & 0.957 & 0.902 \\ 
      3 & 1.208 & 0.025 & 0.595 & 0.947 & 0.897 \\ 
      4 & 0.964 & 0.076 & 0.468 & 0.757 & 0.458 \\ 
      5 & 0.890 & 0.124 & 0.372 & 0.679 & 0.470 \\ 
      6 & 0.886 & 0.147 & 0.385 & 0.646 & 0.362 \\ 
      7 & 0.884 & 0.138 & 0.390 & 0.641 & 0.333 \\ 
      8 & 0.804 & 0.148 & 0.369 & 0.627 & 0.383 \\ 
       \bottomrule
    \end{tabular}
    \caption{Values of $\pmis$ and other cluster validity indices for the Palmer penguins data based on the $k$-means clusterings of the observations.}
    \label{tab:penguins_pmis_stability_kmeans}
\end{table}

\newpage

\begin{table}[ht!]
    \centering
    \begin{tabular}{rrrrrr}
      \toprule
      $K$ & Gap & $\pmis$ & Silhouette & Stability (ARI) & Prediction Strength \\ 
      \midrule
      1 & 0.545 & 0.000 &   --- &   --- &   --- \\ 
      2 & 1.133 & 0.012 & 0.581 & 0.909 & 0.942 \\ 
      3 & 1.325 & 0.024 & 0.595 & 0.943 & 0.899 \\ 
      4 & 1.124 & 0.063 & 0.483 & 0.757 & 0.544 \\ 
      5 & 1.051 & 0.099 & 0.470 & 0.702 & 0.464 \\ 
      6 & 1.024 & 0.141 & 0.382 & 0.652 & 0.413 \\ 
      7 & 1.013 & 0.132 & 0.387 & 0.656 & 0.366 \\ 
      8 & 0.994 & 0.128 & 0.381 & 0.645 & 0.345 \\ 
       \bottomrule
    \end{tabular}
    \caption{Values of $\pmis$ and other cluster validity indices for the Palmer penguins data based on the hierarchical clusterings of the observations.}
    \label{tab:penguins_pmis_stability_hcl}
\end{table}

\newpage

\begin{table}[ht!]
    \centering
    \begin{tabular}{crrrrr}
        \toprule
        $p$ & $\pmis$ & Elapsed (s) & $\hat P_{{\rm mc}}$ & $\sigma(\hat P_{{\rm mc}} )$ & Elapsed (s) \\
        \midrule
        1 & 0.13144 & 0.011 & 0.13143 & 0.00056 & 1.010 \\ 
        2 & 0.13144 & 0.164 & 0.13141 & 0.00041 & 0.885 \\ 
        3 & 0.13144 & 3.163 & 0.13133 & 0.00042 & 0.994 \\ 
        4 & 0.13144 & 23.061 & 0.13132 & 0.00058 & 0.854 \\ 
        5 & 0.13145 & 23.689 & 0.13140 & 0.00051 & 0.954 \\ 
        \bottomrule
    \end{tabular}
    \caption{$\pmis$ values computed using cubature methods and Monte Carlo integration based on 50 replicates. The $\pmis$ and $\hat P_{{\rm mc}}$ values are averages across all replicates. Elapsed time (in seconds) is the average time for a single replicate.}
    \label{tab:mc_cubature_pmis}
\end{table}

\newpage

\header{Supplementary Methods}

\newpage

\section{Monte Carlo $\pmis$ Comparison}

\normalsize

Here we compare $\pmis$ computed using standard cubature methods to $\hat P_{{\rm mc}}$ estimated using a Monte Carlo (MC) integral (as described in Section 4.5). We use 50 replicates to obtain the timing measurements and quantify the uncertainty in the MC integral. All values presented are means over these replicates unless otherwise specified. For the MC integration procedure we use $M = 10^5$ sample points, and parallelize the computation over 8 cores.

The distribution in question consists of three Gaussian clusters with $\pi_k = 1/3$ in $\mathbb{R}^p$ with $I_p$ variance. The cluster are centered at $0^p$ and $\pm d^p$, where $d^p$ is the $p$-dimensional vector whose elements are all $d$. $d$ is set so that the Euclidean distance between $d^p$ and the origin is fixed to be 3; i.e. $d = \sqrt{3^2/p}$. This is so that the value of $\pmis$ remains fixed across dimensions and we do not need to worry about the dimensionality affecting the true value of $\pmis$. The results for dimension $p = 1, \dots, 5$ are presented in Supplementary Table \ref{tab:mc_cubature_pmis}.

Both approaches produce highly similar values of $\pmis$ across dimensions. The standard deviation of the MC estimates is quite low, indicating stability in the estimation procedure. However, while the cubature method evaluation time rapidly increases in $p$, the time for the MC procedure is roughly constant. These together highlight the MC estimation procedure as a viable and accurate approach to estimate $\pmis$, especially in moderate dimensional data (i.e., $p \geq 3$) where cubature methods may struggle to reach a solution in a reasonable time.

\end{document}